\documentclass{article}

\usepackage[preprint, nonatbib]{neurips_2026}

\usepackage[utf8]{inputenc} 
\usepackage[T1]{fontenc}    
\usepackage{hyperref}       
\usepackage{url}            
\usepackage{booktabs}       
\usepackage{amsfonts}       
\usepackage{nicefrac}       
\usepackage{microtype}      
\usepackage{xcolor}         

\usepackage{amssymb}
\usepackage{amsthm}
\usepackage{graphicx}
\usepackage{dsfont}
\usepackage{amsmath}
\usepackage{algorithm}
\usepackage{algpseudocode}
\usepackage{setspace}
\usepackage{multirow}
\usepackage{wrapfig}
\usepackage{xspace}
\usepackage{etoolbox}

\usepackage{enumitem} 
\usepackage[most]{tcolorbox}

\DeclareMathOperator*{\argmax}{arg\,max}
\newcommand\indep{\protect\mathpalette{\protect\independenT}{\perp}}
\def\independenT#1#2{\mathrel{\rlap{$#1#2$}\mkern2mu{#1#2}}}

\usepackage[textsize=tiny]{todonotes} 

\usepackage{pifont}

\newtheorem{theorem}{Theorem}

\newtheorem{remark}{Remark}
\newtheorem{definition}{Definition}
\newtheorem{assumption}{Assumption}

\usepackage[square,numbers]{natbib}
\definecolor{darkgreen}{RGB}{0,100,0}

\title{Amortizing Causal Sensitivity Analysis via Prior Data-Fitted Networks}

\author{%
\begin{tabular}{c}
\textbf{Emil Javurek}\textsuperscript{1,2,}\thanks{Corresponding author: \texttt{emil.javurek@lmu.de}} ,
\textbf{Dennis Frauen}\textsuperscript{1,2},
\textbf{Marie Brockschmidt}\textsuperscript{1,2},
\textbf{Jonas Schweisthal}\textsuperscript{1,2},
\\[-0.1em]
\textbf{Stefan Feuerriegel}\textsuperscript{1,2}
\\[0.7em]
{\normalfont\mdseries
\textsuperscript{1}LMU Munich
\qquad
\textsuperscript{2}Munich Center for Machine Learning (MCML)}
\end{tabular}
}

\begin{document}

\maketitle

\begin{abstract}

Causal sensitivity analysis aims to provide bounds for causal effect estimates in the presence of unobserved confounding. However, existing methods for causal sensitivity analysis are per-instance procedures, meaning that changes to the dataset, causal query, sensitivity level, or treatment require new computation. Here, we instead present an in-context learning approach. Specifically, we propose an amortized approach to causal sensitivity analysis based on prior-data fitted networks. A key challenge is that the sensitivity bounds are not directly available when sampling training data. To address this, we develop a general prior-data construction that is applicable across the class of generalized treatment sensitivity models. Our construction involves a Lagrangian scalarization of the objective to generate training labels for the bounds through a tradeoff between causal effect min/max-imization and sensitivity model violation, which avoids model-specific analytical derivations. We further show that, under standard convexity and linearity conditions, our objective recovers the full Pareto frontier of solutions. Empirically, we demonstrate our amortized approach across various datasets, causal queries, and sensitivity levels, where our approach achieves a test-time computation that is orders of magnitude faster than per-instance methods. To the best of our knowledge, ours is the first foundation model for in-context learning for causal sensitivity analysis.

\end{abstract}

\section{Introduction}
\label{sec:introduction}
\vspace{-0.2cm}



Causal inference from observational data is widely used in fields such as medicine \citep{feuerriegelCausalMachineLearning2024,skapetzeMonitoringChangesVitamin2025}, public policy \citep{kuzmanovicCausalMachineLearning2024}, and the social sciences \citep{barRoleSocialMedia2025}, but relies on untestable assumptions about the presence and strength of unobserved confounding \citep{pearlCausalityModelsReasoning2013}. \textbf{\emph{Causal sensitivity analysis}} addresses this limitation by replacing point estimates with bounds on causal effects in the presence of unobserved confounding \citep{manskiNonparametricBoundsTreatment1989}. In practice, such bounds can be sufficient for decision-making---for example, to assess when conclusions remain unchanged across a range of plausible confounding strengths.


Formally, causal sensitivity analysis computes lower and upper bounds (i.e., $\theta^-$ and $\theta^+$) on a causal estimand as functions of a sensitivity level $\Gamma$. A wide range of methods exist for estimating these bounds \citep{dornSharpSensitivityAnalysis2023,
frauenSharpBoundsGeneralized2023,
frauenNEURALFRAMEWORKGENERALIZED2024a,
jessonQuantifyingIgnoranceIndividualLevel2021,
jessonScalableSensitivityUncertainty2022,
jinSensitivityAnalysisIndividual2023,
jinSensitivityAnalysisSensitivity2026,
kallusIntervalEstimationIndividualLevel2019,
manskiNonparametricBoundsTreatment1989,
marmarelisPartialIdentificationDose2023,
marmarelisEnsembledPredictionIntervals2024,
oprescuBLearnerQuasiOracleBounds2023,
rosenbaumSensitivityAnalysisCertain1987,
tanDistributionalApproachCausal2006,
yinConformalSensitivityAnalysis2024,
zhaoSensitivityAnalysisInverse2019}, but these methods are inherently \textit{per-instance} computations. This means that each change to the dataset, causal query, treatment arm, or sensitivity level $\Gamma$ requires a new computation. The need to recompute the bounds \textit{per-instance} often limits practical use. As a result, sensitivity analysis is frequently applied as a final check, rather than as a systematic tool to comprehensively understand how the conclusions vary across assumptions.


In this paper, we instead propose an \emph{in-context learning} approach to causal sensitivity analysis. Our central idea is to follow an amortized approach and replace the per-instance computation with a learned predictor that, given a dataset, causal query, and sensitivity level, directly returns the corresponding sensitivity bounds. In principle, a natural framework for such amortization is through prior-data fitted networks \citep{Muller.12202021}; however, this is \textit{non-trivial}. In point-identified settings, prior-data construction is straightforward: one samples from a data-generating process, which directly provides the causal effect that can be used as a label. In contrast, causal sensitivity analysis requires lower and upper bounds under a sensitivity model. Importantly, these bounds are not directly available from the data-generating process and must themselves be computed. As a result, the key challenge is therefore \textit{how to construct such prior-data labels in a tractable manner for a broad range of sensitivity models.}

\begin{figure}[t]
    \vspace{-1.0cm}
    \centering
    \includegraphics[width=\linewidth]{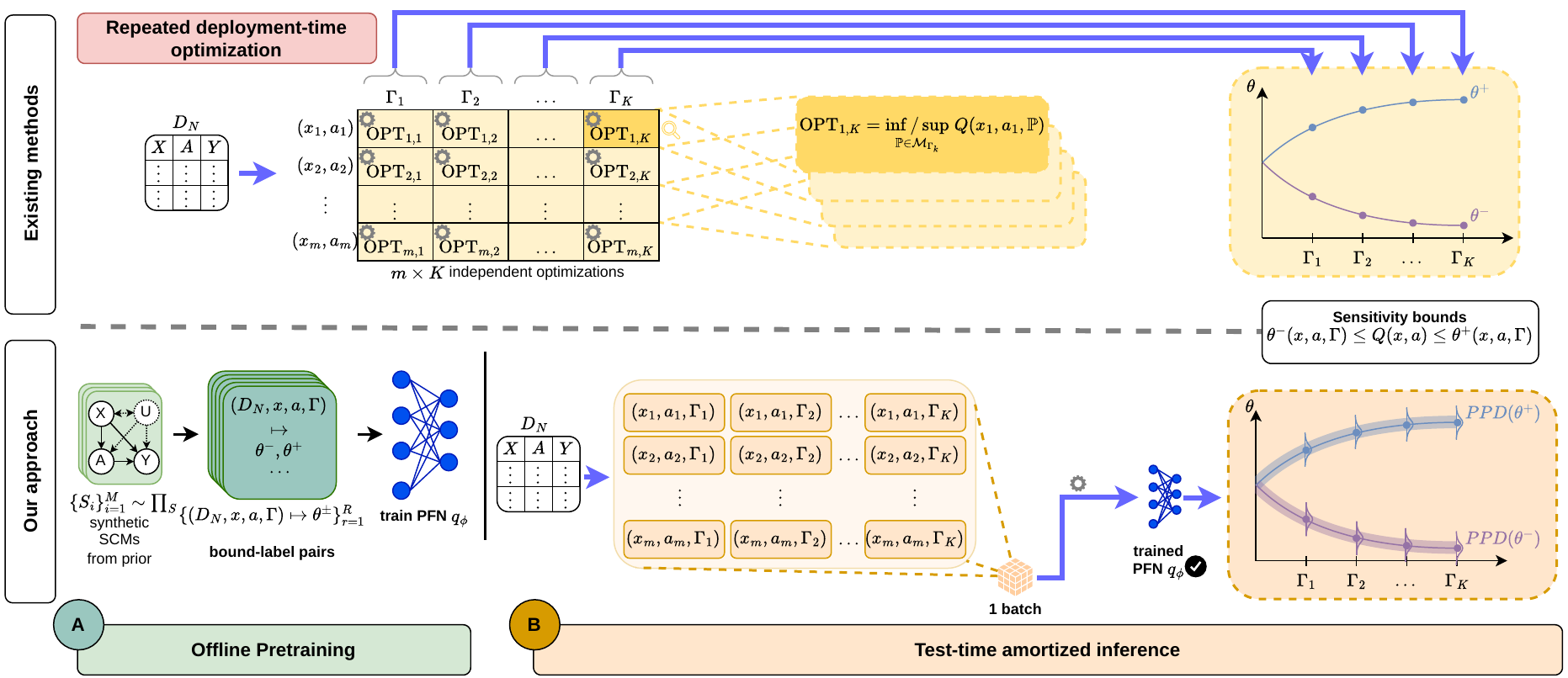}
    \vspace{-0.5cm}
    \caption{
        \textbf{Per-instance optimization vs. amortization.} \footnotesize{Existing methods (top row) perform \emph{per-instance} optimization: for each input query $(x_i,a_i)$ and each level of sensitivity $\Gamma_k$, a new optimization must be instantiated. The sensitivity bound curves (right) are constructed across $m\times K$ optimizations. Our approach (bottom row) amortizes: Expensive pretraining is done once \textit{offline} (step A). Once trained, the PFN processes all input queries in a single \textit{batched} forward pass at trivial cost (step B).
    } 
    }
    \label{fig:1-Intro}
    \vspace{-0.6cm}
\end{figure}


In this paper, we introduce a general prior-data construction for obtaining labels wrt. sensitivity bounds. Our construction is broadly applicable and can be used together with the class of generalized treatment sensitivity models (GTSMs) \citep{frauenNEURALFRAMEWORKGENERALIZED2024a}, which includes the marginal sensitivity model (MSM) \citep{tanDistributionalApproachCausal2006}, $f$-sensitivity models  \citep{jinSensitivityAnalysisSensitivity2026}, and Rosenbaum’s sensitivity model \citep{rosenbaumSensitivityAnalysisCertain1987}. To derive our construction, we formulate the problem of generating sensitivity bounds as a multi-criteria optimization problem over latent distribution shifts to tradeoff between causal effect min/max-imization and sensitivity model violation. By varying the tradeoff, we trace the Pareto frontier of bound values along different sensitivity levels and thereby yield the lower and upper bounds needed for prior-data training. Our approach does not require closed-form analytical solutions and thus applies to all sensitivity models from the GTSM class. We use the construction to train a prior-data fitted foundation model that allows us to predict sensitivity bounds based on new datasets, causal queries, and sensitivity levels.

Finally, we use the resulting labels to train a prior-data fitted \emph{foundation model}. After training, inference reduces to a single forward pass: given a dataset, causal query, and sensitivity level, the foundation model simply predicts the lower and upper sensitivity bounds directly from data without recomputing them from scratch. This enables fast evaluation across many causal queries and assumptions. Because our prior-data construction is not tied to a specific sensitivity model, the approach applies broadly to the class of GTSMs. Finally, the foundation model outputs full posterior predictive distributions for the lower and upper bounds, thereby providing uncertainty-aware estimates rather than only point predictions. To the best of our knowledge, this is the first foundation model for causal sensitivity analysis.


Our \textbf{contributions} are three-fold:\footnote{Code available at: \url{https://github.com/EmilJavurek/Amortizing-Causal-Sensitivity-Analysis-via-PFNs}.}
\vspace{-0.2cm}
\begin{itemize}[leftmargin=6mm, labelsep=0.5em]
    \item[(1)] \textbf{Prior-data label construction.}
    We introduce a general construction scheme for generating training labels with sensitivity bounds. The construction involves a Lagrangian scalarization of the objective, which we show theoretically to recover the full Pareto frontier of solutions. 

    \item[(2)] \textbf{A foundation model for causal sensitivity analysis.}
    We train a prior-data fitted foundation model to offer a new in-context learning approach for causal sensitivity analysis. Our foundation model directly predicts the sensitivity bounds given a dataset, causal query, and sensitivity level.

    \item[(3)] \textbf{Empirical demonstration.}
    We empirically demonstrate that our foundation model can approximate sensitivity bounds across a broad range of settings within a single forward pass. 
\end{itemize}

\section{Related Work}
\label{sec:related-work}
\vspace{-0.2cm}




\textbf{Causal sensitivity analysis:} Causal sensitivity analysis aims at partial identification of causal queries under a sensitivity model. Several sensitivity models have been developed for this purpose, such as the marginal sensitivity model (MSM) \citep{tanDistributionalApproachCausal2006}, the $f$-sensitivity models \citep{jinSensitivityAnalysisSensitivity2026}, and Rosenbaum's sensitivity model \citep{rosenbaumSensitivityAnalysisCertain1987}. Recently, Frauen et al. \citep{frauenSharpBoundsGeneralized2023,frauenNEURALFRAMEWORKGENERALIZED2024a} proposed a unified approach in the form of the generalized treatment sensitivity model (GTSM), which subsumes all three aforementioned sensitivity models.

Existing approaches to estimating sensitivity bounds fall into two broad categories. On the one hand, closed-form sharp bounds are available, but only for a small number of highly structured settings (e.g., the MSM) \citep{dornSharpSensitivityAnalysis2023, frauenSharpBoundsGeneralized2023}. However, this applies only to the MSM and not, for example, the broader set of sensitivity models in the GTSM. On the other hand, GTSMs typically require a two-stage neural density estimation procedure \citep{frauenNEURALFRAMEWORKGENERALIZED2024a}, and are thus computationally expensive. Importantly, all machine learning approaches to estimating sensitivity bounds operate under a \textit{per-instance} approach, meaning that changes to the dataset, causal query, sensitivity level, or treatment require new computations. \textit{As a result, causal sensitivity analysis is computationally complex at deployment, and an approach for in-context learning is still missing.}

\textbf{Prior data-fitted networks:} Prior-data fitted networks (PFNs) \citep{Muller.12202021} are foundation models that amortize Bayesian posterior predictive inference via in-context learning. Hence,  PFNs are trained on synthetic data sampled from a prespecified prior. TabPFN \citep{grinsztajnTabPFN25AdvancingState2026, Hollmann.752022, hollmannAccuratePredictionsSmall2025} demonstrated this paradigm at scale by combining transformers with a Bayesian neural network prior over structural causal models, thereby achieving state-of-the-art performance on tabular prediction. Subsequent work has extended PFNs to other modalities (e.g., such as time-series forecasting \citep{hooTablesTimeExtending2025}) and provided theoretical insights into the in-context learning behavior theoretically \citep{melnychukFrequentistConsistencyPriorData2026, naglerStatisticalFoundationsPriorData2023}. More broadly, PFNs offer a general approach for amortized inference across tasks that would traditionally be solved on a \textit{per-instance} basis. 


\textbf{Foundation models for causal inference:} Only recently, first works have applied PFN-style amortization to causal inference \citep{Balazadeh.692025, bynumBlackBoxCausal2025, Ma.6122025, Robertson.662025}. The central methodological challenge in this line of work is the design of priors over data-generating processes that ensure that the learned model performs valid \textit{causal} inference at deployment, rather than merely interpolating observational patterns. For example, \citet{Ma.6122025} propose a framework for constructing priors based on causal structural models. However, all of the existing approaches focus on causal inference using \textit{point identification}, where all confounders are \emph{observed}. 

In contrast, causal sensitivity analysis focuses on settings where some confounders are \textit{unobserved} and thus aims at \textit{partial identification} by bounding the causal estimand. This makes the problem fundamentally harder: unlike point-identified causal effects, sensitivity bounds are not available as direct functionals of a sampled structural causal model and must themselves be computed via non-trivial optimization procedures. This also makes the development of foundation models for this task substantially more challenging. \textit{To the best of our knowledge, no prior work has developed a foundation model for causal sensitivity analysis; we provide the first principled approach for this.}

\vspace{-0.2cm}
\section{Mathematical Background}
\label{sec:mathematical-background}
\vspace{-0.2cm}


\begin{wrapfigure}{r}{0.25\linewidth}
    \vspace{-0.8cm}
    \centering
    \includegraphics[width=\linewidth]{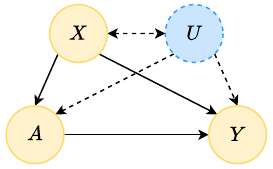}
    \vspace{-0.6cm}
    \caption{
    Causal graph. \footnotesize{Observed variables are colored orange and unobserved blue. We allow for arbitrary dependence between $X$ and $U$.}
    }
    \label{fig:2-causal-graph}
    \vspace{-0.3cm}
\end{wrapfigure}

\textbf{Notation:}
We write random variables in uppercase (e.g., $X$) and their realizations in lowercase (e.g., $x$). We write $\mathbb{P}$ for a probability distribution, with $\mathbb{P}(x)$ denoting the probability mass/density function if $X$ is discrete/continuous. Conditional probability mass/density functions are written as $\mathbb{P}(Y = y \mid X = x)$, or shorthand $\mathbb{P}(y \mid x)$, while a conditional distribution $Y|X=x$ is written as $\mathbb{P}(Y \mid x)$. Expectations are written as $\mathbb{E}_{\mathbb{P}}[\cdot]$, with the subscript omitted when the underlying distribution is clear from the context. We denote by $f_\sharp \mathbb{P}$ the push-forward of $\mathbb{P}$ through $f : \mathcal{A} \mapsto \mathcal{B}$, defined by $(f_\sharp \mathbb{P})(B) = \mathbb{P}({a : f(a) \in B})$ for measurable sets $B \in \mathcal{B}$.

\vspace{-0.2cm}
\subsection{Problem setup}
\label{sec:problem-setup}
\vspace{-0.2cm}

\textbf{Setting:}
We consider a standard causal inference setting with backdoor adjustment \citep{dornSharpSensitivityAnalysis2023}, with observed covariates $X \in \mathcal{X} \subseteq \mathbb{R}^{d_x}$, an unobserved confounder $U \in \mathcal{U} \subseteq \mathbb{R}$, a binary treatment $A \in \mathcal{A} = \{0,1\}$, and an outcome $Y \in \mathcal{Y} \subseteq \mathbb{R}$ (see Fig.~\ref{fig:2-causal-graph}). The data-generating process is formalized as a structural causal model (SCM) $\mathcal{S} = (X, U, A, Y, f, \mathbb{P}_U, \mathbb{P}_\varepsilon)$, where $f = (f_X, f_A, f_Y)$ denotes the structural assignments and $\mathbb{P}_U, \mathbb{P}_\varepsilon$ the distributions over unobserved confounding $U$ and exogenous noise variables $\varepsilon = (\varepsilon_X, \varepsilon_A, \varepsilon_Y)$ \citep{pearlCausalityModelsReasoning2013}. The SCM induces an observational distribution $\mathbb{P}^{\mathcal{S}}_{\mathrm{obs}}$ over $(X,A,Y)$, from which we observe an i.i.d. dataset $D_n = (x_i,a_i,y_i)_{i=1}^n \sim \mathbb{P}^{\mathcal{S}}_{\mathrm{obs}}$. The interventional distributions $\mathbb{P}^{\mathcal{S}}_{\mathrm{int}}$ are obtained by replacing the treatment assignment mechanism with the intervention $\mathrm{do}(A=a)$ \citep{pearlCausalityModelsReasoning2013}. We denote the potential outcome under this intervention by $Y(a)$ \citep{rubinEstimatingCausalEffects1974}. Throughout, the true SCM $\mathcal{S}^\star$ is \emph{unknown}; only samples from $\mathbb{P}^{\mathcal{S}^\star}_{\mathrm{obs}}$ are observed.


\textbf{Causal query:} For a data-generating process $\mathcal{S}$, we are interested in estimating a causal query $Q$ of the form $\mathcal{Q}\left(\mathbb{P}^{\mathcal{S}}_{\mathrm{int}}(Y(a) \mid x)\right)$, where $\mathcal{Q}$ maps the conditional potential outcomes distribution to $\mathcal{Y}$. In this work, we focus on the conditional average potential outcome (CAPO), i.e., $Q(x,a,\mathbb{P}) = \mathbb{E}_{\mathbb{P}}[Y(a) \mid x]$, 
as a primitive causal estimand. Other common estimands, such as the average treatment effect (ATE) and conditional average treatment effects (CATE), can be obtained by simple (post-inference) transformations.\footnote{
For the ATE, one marginalizes over covariates the difference $\mathbb{E}_{X}[Q(x,1,\mathbb{P}) - Q(x,0,\mathbb{P})]$. Similarly, the CATE is given the difference $Q(x,1,\mathbb{P}) - Q(x,0,\mathbb{P})$.
}.

\begin{assumption}\label{assumption1}
We work under the standard causal assumption \citep{dornSharpSensitivityAnalysis2023}: (i) consistency: \mbox{$A=a \Rightarrow Y(a)=Y, \ \forall a$}, (ii) positivity: \mbox{$\mathbb{P}(a|x) > 0, \ \forall (x,a)$}, and (iii) latent unconfoundedness: \mbox{$Y(a) \indep A \mid X, U, \ \forall a$}. 
\end{assumption}

\textbf{Partial identification:} Because $U$ is unobserved, the causal query $Q(x,a,\mathbb{P})$ is generally \textbf{not} point-identified from the observed data distribution $\mathbb{P}^{\mathcal{S}^\star}_{\mathrm{obs}}$. In particular, there may exist different data-generating processes $\mathcal{S} \neq \mathcal{S}^\star$ that induce the \emph{same} observational distribution $\mathbb{P}^{\mathcal{S}}_{\mathrm{obs}} = \mathbb{P}^{\mathcal{S}^\star}_{\mathrm{obs}}$ over $(X,A,Y)$ but imply different causal effects $Q(\cdot, \mathbb{P}) \neq Q(\cdot, \mathbb{P}^\star)$. As a result, the causal query is only partially identified in the sense that we can only infer a set of causal effects compatible with the observed data distribution and the underlying assumptions \citep{manskiNonparametricBoundsTreatment1989}.

\vspace{-0.2cm}
\subsection{Causal sensitivity analysis}
\label{sec:sensitivity-analysis}

Causal sensitivity analysis introduces a structured way to characterize the set of data-generating processes compatible with the observed data by parameterizing the strength of unobserved confounding through a sensitivity model. 

\begin{definition}
    A sensitivity model $\mathcal{M}_{\Gamma}$ is a family of distributions $\mathbb{P}$ over $(X,U,A,Y)$ parameterized by a sensitivity level $\Gamma \in  \mathbb{R}_{\geq \Gamma_{\mathrm{min}} \geq 0}$, such that $\int_{\mathcal{U}}\mathbb{P}(x,u,a,y)\mathrm{d}u \equiv \mathbb{P}_{\mathrm{obs}} = \mathbb{P}_{\mathrm{obs}}^\star$ for all $\mathbb{P} \in \mathcal{M}_{\Gamma}$.
\end{definition}

In this work, we focus on the so-called \emph{generalized treatment sensitivity model (GTSM)}, which is a broad class of sensitivity models that constrain the latent distribution shift induced by treatment intervention \citep{frauenNEURALFRAMEWORKGENERALIZED2024a}. Under the GTSM, the admissible full distributions satisfy constraints of the form
\begin{align}
    \Delta_{x,a}\left(
        \mathbb{P}(U \mid x),
        \mathbb{P}(U \mid x,a)
    \right)
    \leq \Gamma, \ \forall (x,a)
\end{align}
where $\Delta_{x,a}$ is a divergence functional measuring the discrepancy between the latent distribution under intervention (i.e., $\mathbb{P}(U \mid x)$)\footnote{
We recognize that $\mathbb{P}(U \mid x, \mathrm{do}(A=a)) = \mathbb{P}(U \mid x)$ since $A \rightarrow U$ is the anti-causal direction. Unobserved confounding $U$ \textit{causes} treatment $A$ (rather than the other way) and so cannot be influenced by an intervention on the treatment. Rather, the intervention removes the dependence between $A$ and $U$.
} and the latent distribution observed under treatment assignment (i.e., $\mathbb{P}(U \mid x,a)$). The GTSM subsumes several widely used sensitivity models, including the marginal sensitivity models \citep{tanDistributionalApproachCausal2006}, the $f$-sensitivity models \citep{jinSensitivityAnalysisSensitivity2026}, and Rosenbaum's sensitivity model \citep{rosenbaumSensitivityAnalysisCertain1987}, depending on the choice of divergence $\Delta_{x,a}$ (see Appendix~\ref{subsec:gtsm-divergences} for details). 

\begin{tcolorbox}[
    colback=white, 
    colframe=black, 
    sharp corners, 
    boxrule=0.8pt, 
    colbacktitle=black, 
    coltitle=white,
    fonttitle=\bfseries
]
\textbf{Task:} For a given sensitivity model $\mathcal{M}_{\Gamma}$, the aim of causal sensitivity analysis is to obtain the sensitivity bounds for the causal query:
\begin{align}
    \theta^-(x,a,\Gamma)
    =
    \inf_{\mathbb{P} \in \mathcal{M}_\Gamma}
    Q(x,a,\mathbb{P}),
    \qquad
    \theta^+(x,a,\Gamma)
    =
    \sup_{\mathbb{P} \in \mathcal{M}_\Gamma}
    Q(x,a,\mathbb{P}).
\end{align}
\end{tcolorbox}
By definition, the interval $[\theta^-(x,a,\Gamma), \theta^+(x,a,\Gamma)]$ is the tightest interval of causal query values compatible with the observed distribution $\mathbb{P}_{\mathrm{obs}}^\star$ and the sensitivity level $\Gamma$. In other words, 
\begin{align}
    Q^\star \in [\theta^-(x,a,\Gamma), \theta^+(x,a,\Gamma)]
\end{align}
is the maximum information about $Q^\star$ that we can extract from $D_n$ as $n \rightarrow \infty$, given $\mathbb{P}^\star \in  \mathcal{M}_{\Gamma}$. The parameter $\Gamma$ controls the width of the interval; when no unobserved confounding is permitted, the causal query is point identified. 

For ease of reading, we omit the sensitivity model $\mathcal{M}_{\Gamma}$ from notation during the rest of the paper, referring to an arbitrary sensitivity model from GTSM, unless stated otherwise.

\vspace{-0.2cm}
\subsection{Background on PFNs}
\label{sec:background-on-pfns}


Prior-data fitted networks (PFNs) are neural foundation models trained to amortize posterior predictive inference across data-generating processes \citep{Muller.12202021}. Given a dataset $D_N \sim \mathbb{P}$, a trained PFN predicts the target distribution $\mathbb{P}(Y \mid X)$ without task-specific gradient updates or retraining (i.e., in-context). Training proceeds by sampling data-generating processes over a prior $\Pi$ and minimizing the negative log-likelihood
\begin{align}
\vspace{-0.5cm}
    \mathcal{L}(\theta)
    =
    \mathbb{E}_{\mathbb{P} \sim \Pi_{\mathcal{S}}}
    \mathbb{E}_{N \sim \Pi_N}
    \mathbb{E}_{D_N \sim \mathbb{P}^N}
    \mathbb{E}_{(X,Y) \sim \mathbb{P}}
    \left[
        -\log q_{\theta}(Y \mid X, D_N)
    \right] .
\end{align}
This yields the predictor $q_{\theta}(Y \mid X, D_N)$ approximating the posterior predictive distribution (PPD) 
\begin{equation}
    \Pi(Y \mid X, D_N)
    =
    \int \mathbb{P}(Y \mid X)\Pi(\mathbb{P} \mid D_N)\,\mathrm{d}\mathbb{P}.
\end{equation}
Because the prior-data distribution can be entirely synthetic, PFNs enable foundation-model-style pretraining for broad classes of tabular prediction problems.

\vspace{-0.2cm}
\section{Amortized Causal Sensitivity Analysis with PFNs}
\label{sec:amortized-sensitivity-analysis}


\textbf{Overview:} We develop a PFN-based approach for amortized causal sensitivity analysis, with the aim of replacing per-instance optimization with a foundation model that predicts sensitivity bounds in a single forward pass for new datasets, causal queries, and sensitivity levels. We first formalize this amortized prediction task and describe the \textbf{high-level learning setup} in Section~\ref{sec:formalized-task}, and the \textbf{synthetic data generation via SCM priors} in Section~\ref{sec:prior-over-dgps}. The central difficulty, however, is label construction: sensitivity bounds lack closed-form expressions in general, and existing solutions are not built for PFN-scale training label generation. To overcome this bottleneck, Section~\ref{sec:label-construction} introduces our core technical contribution: a general \textbf{prior-label construction procedure} based on Lagrangian scalarization, which allows us to efficiently generate sensitivity bound labels across sampled SCMs for training PFNs.

\vspace{-0.2cm}
\subsection{Task formulation}
\label{sec:formalized-task}
\vspace{-0.1cm}

\textbf{PFN:} To adapt PFNs to causal sensitivity analysis under a fixed sensitivity model $\mathcal{M}$, we train a PFN to learn the mapping
\begin{align}
    D_{N}, x, a, \Gamma \longrightarrow \mathrm{PPD}(\theta^{-}),\mathrm{PPD}(\theta^{+}) ,
\end{align}
where $\mathrm{PPD}(\cdot)$ denotes the posterior predictive distribution induced by the prior over data-generating processes In other words, for a dataset $D_{N}$, at a specific query $(x,a)$, and under unobserved confounding constrained by $\Delta_{x,a} \leq \Gamma$ according to $\mathcal{M}$, we learn the PPD over the lower and upper bounds $\theta^-$ and $\theta^+$ of the causal query $Q(x,a) = \mathbb{E}_{\mathbb{P}}[Y(a) \mid x]$. 

\textbf{Training data:} During training, we sample a large, diverse collection of data-generating processes $\{\mathcal{S}_i\}_{i=1}^{n_i} \sim \Pi_{\mathcal{S}}$, each of which generates a set of input tuple $(D_{N}, x, a, \Gamma)$ and output label $(\theta^{-},\theta^{+})$ supervised training pairs $\{(D_{N}, x, a, \Gamma)_j$, i.e., $(\theta^{-},\theta^{+})_j\}_{j=1}^{n_j}$. By constructing a sufficiently rich prior-data distribution, the PFN trained on $n_i \times n_j$ such pairs is able to amortize sensitivity analysis across datasets, causal queries, and sensitivity levels.\footnote{Throughout, we omit the SCM index $\mathcal{S}$ on all training pairs $\{(D_{N}, x, a, \Gamma),(\theta^{-},\theta^{+})\}$ for ease of exposition.}

An alternative to our formulation is to consider for output a PPD on the causal query $Q$ \emph{directly} and to obtain bounds by taking a credible interval. While technically feasible, it is fundamentally misaligned with the nature of the setting: since the true causal query $Q^\star$ is only partially identified under the sensitivity model, the maximum information about $Q^\star$ that we can extract from the data $D_n$ as $n \rightarrow \infty$ are the bounds $\theta^-,\theta^+$, i.e. $Q^\star \in [\theta^-,\theta^+]$. The distribution of $Q^\star$ \emph{within} this interval \textit{cannot} be learned from the data and is \textit{purely} driven by the prior construction. As such, we focus directly on the identifiable and thus learnable quantities, namely, the lower and upper bounds.

\subsection{Synthetic data generation via SCM priors}
\label{sec:prior-over-dgps}

\textbf{SCM sampling:} We construct the prior-data distribution by sampling SCMs $\mathcal{S}_i \sim \Pi_{\mathcal{S}}$. Each sampled SCM $    \mathcal{S}_i = (X,U,A,Y,f,\mathbb{P}_U,\mathbb{P}_{\varepsilon})
$
induces an observational distribution $\mathbb{P}^{\mathcal{S}}_{\mathrm{obs}}$ over $(X,A,Y)$, from which we sample observational datasets $D_N \sim (\mathbb{P}^{\mathcal{S}}_{\mathrm{obs}})^N$. Following prior-data constructions for other foundation models \citep{Hollmann.752022, Ma.6122025}, we parameterize the structural assignments $f=(f_X,f_A,f_Y)$ using randomly sampled Bayesian neural networks. For example, the outcome is generated by $Y \sim f_{Y,\mathrm{BNN}}(X,U,A,\varepsilon_Y)$, with $f_{Y,\mathrm{BNN}} \sim \Pi_{f_Y}$. This yields a flexible prior over nonlinear covariate distributions, treatment assignment mechanisms, and outcome functions. The distributions for noise and unobserved confounding are sampled from simple parametric families with randomly drawn scales. We provide implementation details in Appendix~\ref{app:implementation-details}.

For effective amortization, we aim to construct a prior $\Pi_\mathcal{S}$ that is as ``broad'' as possible for the setting. In contrast to prior constructions for point-identified causal inference, where the sampled SCMs must be restricted to ensure (point) identifiability of the causal query \citep{Ma.6122025}, our setting is only partially identified due to unobserved confounding. As a result, our prior explicitly allows for SCMs with varying degrees of unobserved confounding, up to the level $\Gamma$ specified by the sensitivity model $\mathcal{M}_\Gamma$.

\textbf{Challenge ($\rightarrow$ label construction).} For each sampled SCM $\mathcal{S}_i$, the prior $\Pi_\mathcal{S}$ provides an observational dataset $D_N$ and query points $(x,a)_j$, with sensitivity levels $\Gamma_j \sim \Pi_\Gamma$. Crucially, however, the corresponding sensitivity bound labels $\theta^\pm(\mathcal{S}_i;x_j,a_j,\Gamma_j)$ are \textbf{not} directly available from the data-generating process $\mathcal{S}_i$; instead, they must be computed. Generating these labels efficiently across many sampled SCMs, queries, and sensitivity levels is addressed in the next section.

\vspace{-0.2cm}
\section{Sensitivity Bound Label Construction via Lagrangian Scalarization}
\label{sec:label-construction} 

To construct training labels, we must compute the sensitivity bounds $\theta^-$ and $\theta^+$ for each sampled SCM and input tuple. We first formulate this as an optimization problem with forward, backward, and bi-objective forms (Section~\ref{sec:directions-of-optimization}). We then introduce a Lagrangian scalarization whose sweep traces the resulting Pareto frontier of bound values (Section~\ref{sec:constructing-the-lagrangian}). Finally, we reduce the optimization over full distributions to a tractable computation over a conditional latent distribution, which yields a scalable label-generation procedure that we use for prior-data training (Section~\ref{sec:instantiated-label-computation}).

\vspace{-0.2cm}
\subsection{Optimization problem for sensitivity bounds}
\label{sec:directions-of-optimization}

\textbf{Forward formulation:} For a sampled SCM $\mathcal{S}_i$ and input $(x_j,a_j,\Gamma_j)$, the desired label is given by the solution of a population-level partial identification problem:
\begin{align}
    \theta^-_j(\mathcal{S}_i;x_j,a_j,\Gamma_j)
    &=
    \inf_{\mathbb{P}
    \in 
    \mathcal{P}(\mathcal{S}_i)
    }
    \mathcal{Q}(\mathbb{P}(Y(a_j)\mid x_j))
    \qquad
    \ \mathrm{s.t.} \
    \Delta_{x_j,a_j}(\mathbb{P}) \leq \Gamma_j    
    ,
    \label{eq:label-forward-bound-lower} \\
    \theta^+_j(\mathcal{S}_i;x_j,a_j,\Gamma_j)
    &=
    \sup_{\mathbb{P}
    \in 
    \mathcal{P}(\mathcal{S}_i)
    }
    \mathcal{Q}(\mathbb{P}(Y(a_j)\mid x_j))
    \qquad
    \ \mathrm{s.t.} \
    \Delta_{x_j,a_j}(\mathbb{P}) \leq \Gamma_j    
    ,
    \label{eq:label-forward-bound-upper}
\end{align}
where $\mathcal{P}(\mathcal{S}_i) = \{\mathbb{P} : \mathbb{P}_{\mathrm{obs}} = \mathbb{P}_{\mathrm{obs}}^{\mathcal{S}_i}\}$ denotes the set of full distributions over $(X,U,A,Y)$ that induce the same observational distribution as $\mathcal{S}_i$. This corresponds to the standard ``forward'' formulation of sensitivity analysis: for a fixed sensitivity level $\Gamma$, one optimizes the causal query subject to observational compatibility and the sensitivity constraint to find the bound.

\textbf{Backward formulation:} The same problem can be approached via a ``backward'' formulation: given a fixed causal query $Q$, minimize $\Gamma$ such that the causal query \emph{becomes the bound} $Q=\theta^\pm$ at the $\Gamma$-level constraint, subject to observational compatibility:
\begin{align}
\Gamma_j^\downarrow(\mathcal{S}_i;x_j,a_j,\theta^{-}_j)
    &= 
    \inf_{
        \mathbb{P} \in \mathcal{P} 
        }
    \Delta_{x_j,a_j}(\mathbb{P})
    \qquad \ \mathrm{s.t.} \
    \mathcal{Q}(\mathbb{P}) \leq \theta_j^{-}
    \ \mathrm{and} \
    \mathbb{P}_{\mathrm{obs}} = \mathbb{P}_{\mathrm{obs}}^{\mathcal{S}_i}
    \\
    \Gamma_j^\uparrow(\mathcal{S}_i;x_j,a_j,\theta^{+}_j)
    &= 
    \inf_{
        \mathbb{P} \in \mathcal{P} 
        }
    \Delta_{x_j,a_j}(\mathbb{P})
    \qquad \ \mathrm{s.t.} \
    \mathcal{Q}(\mathbb{P}) \geq \theta_j^{+}
    \ \mathrm{and} \
    \mathbb{P}_{\mathrm{obs}} = \mathbb{P}_{\mathrm{obs}}^{\mathcal{S}_i} 
\end{align}

\textbf{Bi-objective program:} The forward problem and the backward problem above are two epsilon-constraint reductions of the same underlying bi-objective program
\begin{align}
\max_{\mathbb{P} \in \mathcal{P}(\mathcal{S}_i)} 
\left(\mathcal{Q}(\mathbb{P}(Y(a_j)\mid x_j)), -\Delta_{x_j,a_j}(\mathbb{P})\right).\label{eq:bi-objective}
\end{align}
Maximizing the causal query and minimizing the sensitivity divergence are competing objectives: tighter compatibility with the no-confounding reference $\Delta_{x_j,a_j} \downarrow \Gamma_{\mathrm{min}}$ pins $\mathbb{P}$ near the point-identified value of $\mathcal{Q}$, while permitting larger $\Delta_{x_j,a_j}$ enlarges the feasible set and extends the attainable range of $\mathcal{Q}$. The set of non-dominated solutions of Eq.~\eqref{eq:bi-objective} forms a Pareto frontier in the $(\Gamma, \theta)$-plane, characterized by the maps 
\begin{align}
&\Gamma \mapsto \theta_j^-(\mathcal{S}_i;x_j,a_j,\Gamma) 
\quad \text{(convex, non-increasing)}, \\
&\Gamma \mapsto \theta_j^+(\mathcal{S}_i;x_j,a_j,\Gamma)
\quad \text{(concave, non-decreasing)}.
\end{align}
The forward and backward problems trace these same frontiers from opposite parametrizations and are inverse functions of one another wherever the maps are strictly monotone.

\vspace{-0.2cm}
\subsection{Constructing the Lagrangian}
\label{sec:constructing-the-lagrangian}
\vspace{-0.2cm}

\textbf{Lagrangian formulation:} The two epsilon-constraint reductions both have a common Lagrangian. Weighting the two objectives by a Lagrange multiplier $\lambda > 0$ gives, for the lower and upper bounds respectively,
\begin{align} 
\mathcal{L}_\lambda^{\downarrow}(\mathbb{P};\mathcal{S}_i,x_j,a_j) 
&= -\mathcal{Q}(\mathbb{P}(Y(a_j)\mid x_j)) - \lambda  \Delta_{x_j,a_j}(\mathbb{P}), \label{eq:lagrangian-lower} \\
\mathcal{L}_\lambda^{\uparrow}(\mathbb{P};\mathcal{S}_i,x_j,a_j) 
&= \mathcal{Q}(\mathbb{P}(Y(a_j)\mid x_j)) - \lambda \Delta_{x_j,a_j}(\mathbb{P})
. \label{eq:lagrangian-upper}
\end{align}
For each fixed $\lambda$, the optima $\mathbb{P}^{\star\downarrow}_\lambda$ and $\mathbb{P}^{\star\uparrow}_\lambda$ given by
\begin{align}
\mathbb{P}^{\star\downarrow}_\lambda = 
\argmax_{\mathbb{P} \in \mathcal{P}(\mathcal{S}_i)}
\mathcal{L}_\lambda^{\downarrow}(\mathbb{P};\mathcal{S}_i,x_j,a_j),  \\
\qquad
\mathbb{P}^{\star\uparrow}_\lambda = 
\argmax_{\mathbb{P} \in \mathcal{P}(\mathcal{S}_i)}
\mathcal{L}_\lambda^{\uparrow}(\mathbb{P};\mathcal{S}_i,x_j,a_j)
\end{align}
each yield a single point $(\Gamma^\star, \theta^\star)$ on their respective Pareto frontier:
\begin{align}
\Gamma^{\star\downarrow}(\lambda)
= \Delta_{x_j,a_j}(\mathbb{P}^{\star\downarrow}_\lambda),
&\qquad 
\theta^{\star\downarrow}(\lambda)
= \mathcal{Q}(\mathbb{P}^{\star\downarrow}_\lambda(Y(a_j)\mid x_j)),
\\
\Gamma^{\star\uparrow}(\lambda) 
= \Delta_{x_j,a_j}(\mathbb{P}^{\star\uparrow}_\lambda),
&\qquad
\theta^{\star\uparrow}(\lambda) 
= \mathcal{Q}(\mathbb{P}^{\star\uparrow}_\lambda(Y(a_j)\mid x_j)). 
\end{align}
For both bounds, sweeping $\lambda$ from large to small traces the Pareto frontier from tight (i.e., $\Gamma^\star$ small, bounds tight) to wide (i.e., $\Gamma^\star$ large, bounds wide).

\textbf{Sweeping the frontier:} 
To construct the label, we leverage the Lagrangian formulation together with a warm-start optimization along a sweep over values of $\lambda$. Solving Eq.~\eqref{eq:lagrangian-lower} or Eq.~\eqref{eq:lagrangian-upper} at a fixed $\lambda$ yields a point $(\theta^\star(\lambda), \Gamma^\star(\lambda))$ on the Pareto frontier. Because adjacent $\lambda$ values produce nearby optima, each successive optimization can be initialized with the previous solution and converges in only a small number of additional steps, thereby shrinking computation cost across the sweep. This is a significant advantage over the traditional forward formulation: a grid of forward solves at fixed $\Gamma$ targets has no analogous transferable state and each optimization is forced to cold-start. The following theorem formalizes the connection between the Lagrangian sweep and the Pareto frontier.

\begin{figure}[t]
    \vspace{-.5cm}
    \centering
    \includegraphics[width=0.9\linewidth]{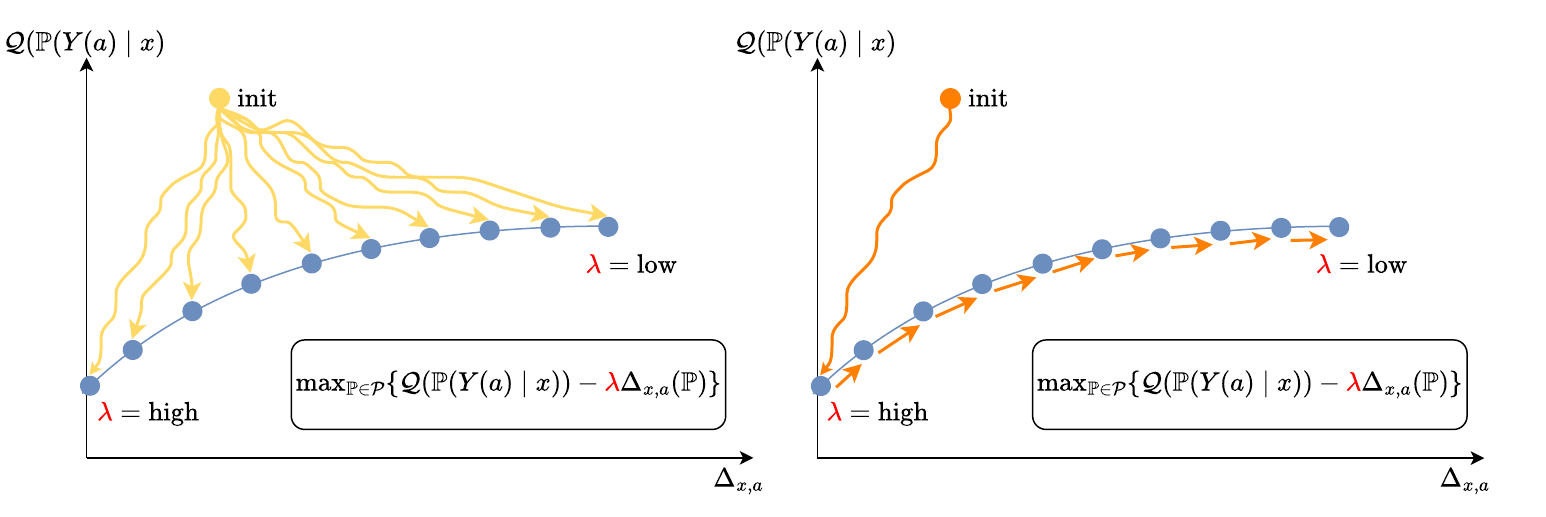}
    \vspace{-0.5cm}
    \caption{
    \textbf{Cold vs warm start.} \scriptsize{Cold-started optimization (left) is re-initialized for each optimization as $\lambda$ varies. Warm-starting (right) finds the Pareto frontier once and then walks across, starting the next optimization where the previous ended.
    }
    }
    \label{fig:3-Pareto-frontier}
    \vspace{-0.3cm}
\end{figure}

\begin{theorem}[Lagrangian sweep recovers the Pareto frontier]
\label{thm:lagrangian-sweep}
Assume (A1) the divergence $\Delta_{x_j,a_j}(\mathbb{P})$ is convex in $\mathbb{P}$ on $\mathcal{P}(\mathcal{S}_i)$, and (A2) the causal query $\mathcal{Q}(\mathbb{P}(Y(a_j) \mid x_j))$ is linear in $\mathbb{P}$ on $\mathcal{P}(\mathcal{S}_i)$. Then, the following is true:

\begin{enumerate}
\vspace{-0.2cm}
    \item[(i)] \textbf{Frontier coverage.} For every $\Gamma \geq \Gamma_{\min}$, there exists $\lambda(\Gamma) \geq 0$ such that any maximizer $\mathbb{P}_{\lambda(\Gamma)}^{\star\uparrow}$ satisfies $\Delta_{x_j,a_j}(\mathbb{P}_{\lambda(\Gamma)}^{\star\uparrow}) = \Gamma$ and $\mathcal{Q}(\mathbb{P}_{\lambda(\Gamma)}^{\star\uparrow}(Y(a_j) \mid x_j)) = \theta_j^+(\mathcal{S}_i; x_j, a_j, \Gamma)$.
\vspace{-0.2cm}
    \item[(ii)] \textbf{Smooth traversal.} The map $\lambda \mapsto \Gamma^{\star\uparrow}(\lambda) := \Delta_{x_j, a_j}(\mathbb{P}_\lambda^{\star\uparrow})$ is monotone non-increasing, and $\lambda \mapsto (\Gamma^{\star\uparrow}(\lambda), \theta^{\star\uparrow}(\lambda))$ traces the upper frontier continuously as $\lambda$ varies.
\vspace{-0.2cm}
\end{enumerate}
The analogous statements hold for the lower bound with
$\mathcal{L}_\lambda^{\downarrow}(\mathbb{P}) = -\mathcal{Q}(\mathbb{P}(Y(a_j) \mid x_j)) - \lambda \Delta_{x_j,a_j}(\mathbb{P})$.
\end{theorem}
\vspace{-0.4cm}
\begin{proof}
    See Appendix~\ref{proof:lagrangian-sweep}
\end{proof}
\vspace{-0.3cm}

\begin{remark}[Assumptions A1 and A2 hold in our setting]
\label{rem:gtsm-convexity}
Assumption (A1) is satisfied for the marginal sensitivity model, $f$-sensitivity models with convex $f$ (i.e., $\mathrm{KL}$, $\chi^2$, total variation, Hellinger), and Rosenbaum's sensitivity model$^\star$. Assumption (A2) holds for CAPO, CATE, and ATE, as each can be expressed as an expectation and, hence, is a linear functional of $\mathbb{P}$.
\end{remark}
\vspace{-0.5cm}
\begin{proof}
    See Appendix~\ref{proof:gtsm-convexity}
\end{proof}
\vspace{-0.3cm}

The above theorem guarantees smooth and complete traversal of the Pareto frontier. This structural property makes warm-starting effective in practice: adjacent values of $\lambda$ produce optima that lie close together on the frontier, so a previously computed solution $\mathbb{P}^\star_\lambda$ provides a high-quality initialization for the subsequent optimization.

\vspace{-0.3cm}
\subsection{Instantiated label computation}
\label{sec:instantiated-label-computation}
\vspace{-0.2cm}

The remaining task is to instantiate the optimization problems in Eq.~\eqref{eq:lagrangian-lower} and Eq.\eqref{eq:lagrangian-upper} in a form that allows for large-scale computation across the prior. To this end, we reduce the abstract objective over $\mathbb{P} \in \mathcal{P}(\mathcal{S}_i)$ to a tractable optimization over a single conditional latent distribution.

First, existing GTSM theory \citep{frauenNEURALFRAMEWORKGENERALIZED2024a} shows that the optimization over the full distributions $\mathbb{P} \in \mathcal{P}(\mathcal{S}_i)$ can be reduced to an optimization over a latent distribution $\textcolor{darkgreen}{\mathbb{P}(U \mid x_j,A\neq a_j)}$. For our objective, this translates to
\begin{align}
    &\max_{\textcolor{blue}{\mathbb{P}(U \mid x_j)}}
    \mathcal{Q}\left(
    f_Y^{\mathcal{S}_i}(x_j,\cdot,a_j)_{\sharp}
    \textcolor{blue}{\mathbb{P}(U \mid x_j)}\right)
    -
    \lambda  
    \Delta_{x_j,a_j}
    \left(
    \textcolor{blue}{\mathbb{P}(U \mid x_j)},
    \mathbb{P}(U \mid x_j, a_j)
    \right),
    \label{eq:operational-objective} \\
    &\mathrm{where } \
    \textcolor{blue}{\mathbb{P}(U \mid x_j)}
    = \pi^{\mathcal{S}_i}(a_j \mid x_j)\mathbb{P}(U \mid x_j, a_j)
    + \left(1 - \pi^{\mathcal{S}_i}(a_j \mid x_j)\right)\,
    \textcolor{darkgreen}{\mathbb{P}(U \mid x_j,A\neq a_j)}
    . \label{eq:candidate-mixture}
\end{align}
We can fit $\textcolor{darkgreen}{\mathbb{P}(U \mid x_j,A\neq a_j)}$ fit with an unconstrained conditional normalizing flow (CNF) \citep{winklerLearningLikelihoodsConditional2019} or any other conditional density estimator of choice. The reparameterization construction in  Eq.~\eqref{eq:candidate-mixture} maps the latent distribution $\textcolor{darkgreen}{\mathbb{P}(U \mid x_j,A\neq a_j)}$ to $\textcolor{blue}{\mathbb{P}(U \mid x_j)}$ and guarantees the observational compatibility, i.e., $\mathbb{P}_{\mathrm{obs}} = \mathbb{P}_{\mathrm{obs}}^{\mathcal{S}_i}$. We provide detailed derivation and intuition of these results in Appendix~\ref{app:theoretical-details}. 

Since we generate data synthetically, we obtain one additional advantage: $f_Y$ is directly available to us and does not need to be estimated (This is unlike in \citet{frauenNEURALFRAMEWORKGENERALIZED2024a} where a second CNF fitting stage is needed).


\vspace{-0.3cm}
\section{Experiments}
\label{sec:experiments}
\vspace{-0.2cm}

The main goal of our experiments is to empirically validate the proposed theoretical construction and assess whether the amortized model accurately recovers sensitivity bounds across different settings. Because \emph{no} existing method supports zero-shot inference for causal sensitivity analysis, we use synthetic datasets to benchmark against the known oracle solutions (where available) and thus to analyze accuracy, robustness, and the sources of performance gains.


\begin{figure}[h]
    \vspace{-0.2cm}
    \centering
    \includegraphics[width=\linewidth]{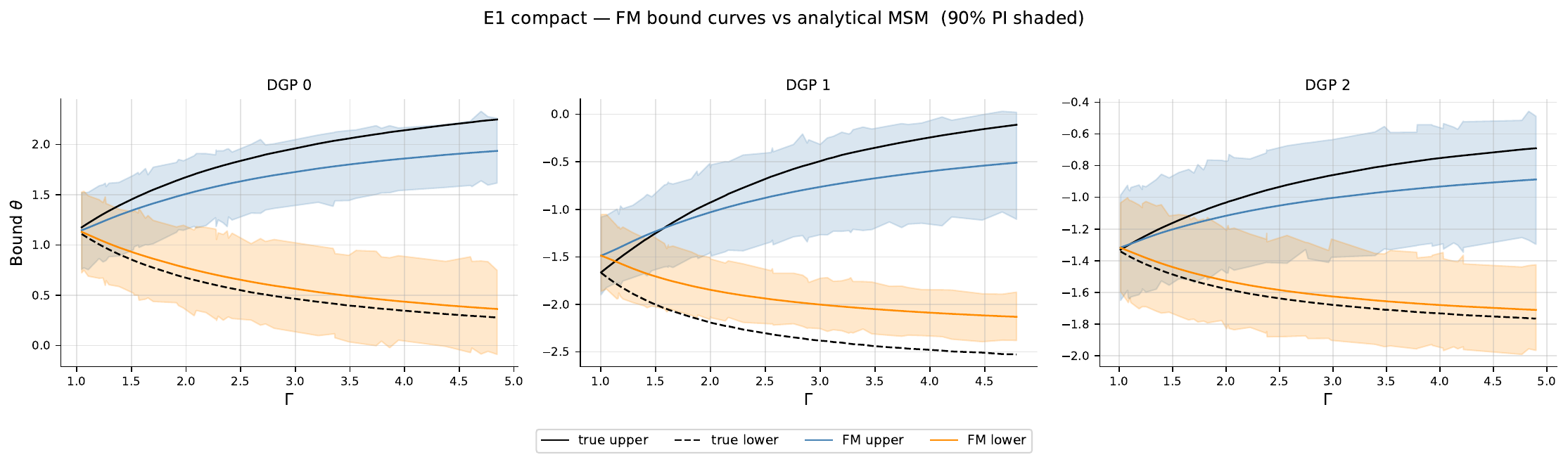}
    \vspace{-0.5cm}
    \caption{\footnotesize{
    \textbf{Example predictions:} $90\%$ posterior predictive intervals for lower and upper bounds for the MSM sensitivity model on three example DGPs. Analytically derived true bounds are shown in black.}}
    \label{fig:example-dgps}
    \vspace{-0.3cm}
\end{figure}

\vspace{-0.2cm}
\subsection{Implementation}
\label{sec:implementation}
\vspace{-0.2cm}

We construct our foundation model (FM) for sensitivity analysis by first sampling the synthetic prior, including the labels as described in Section~\ref{sec:label-construction}. We then train a PFN with two output heads to produce PPDs over the lower and upper bounds. Each head is trained by negative log-likelihood against the labels $\theta^{\star\downarrow}(\lambda)$ and $\theta^{\star\uparrow}(\lambda)$ from Section~\ref{sec:label-construction}. We additionally impose a soft monotonicity regularizer: for fixed $(D_N,x_j,a_j)$, the predicted lower/upper bound should be non-increasing/decreasing in $\Gamma$. The regularizer applies a zero-margin ReLU penalty to adjacent $\Gamma$-sorted predictive means within each $\{(D_N,x_j,a_j,\cdot)\}$ group that violate the expected nesting of the bounds.\footnote{
Empirically, the penalty only activates at the beginning of training, with no monotonicity enforcement needed thereafter, which is desirable.
} At test time, a single batched forward pass over datasets, query points, treatment arms, and sensitivity levels returns the corresponding lower- and upper-bound PPDs; credible intervals are obtained as posterior predictive quantiles. Architecture details, training hyperparameters, and prior-generation runtime are reported in Appendix~\ref{app:implementation-details}.

\vspace{-0.3cm}
\subsection{Sweep evaluation}
\label{sec:sweep-evaluation}
\vspace{-0.1cm}

\begin{wrapfigure}{r}{0.7\linewidth}
    \vspace{-0.8cm}
    \centering
    \includegraphics[width=\linewidth]{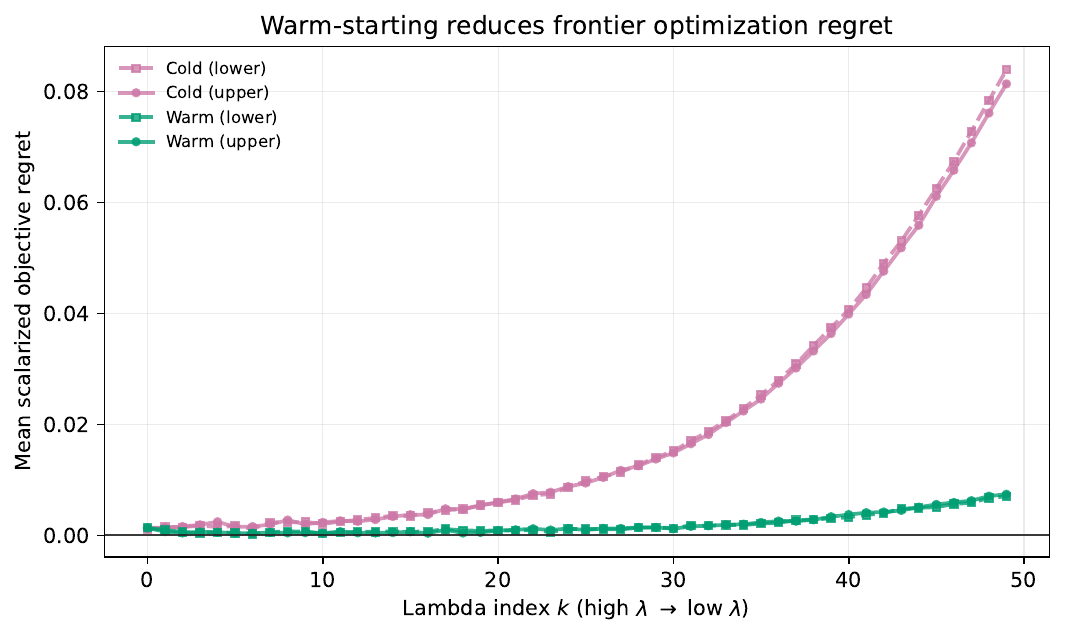}
    \vspace{-0.7cm}
    \caption{\footnotesize{\textbf{Warm start evaluation.} Mean scalarized objective regret along the $\lambda$-sweep ($k=0$ at $\lambda_{\mathrm{max}}=2.0$, $k=49$ at $\lambda_{\mathrm{min}}=0.08$.) measured against a high-budget reference (1000 steps). $\Rightarrow$ \emph{Warm starting achieves lower regret solutions while $1.90\times$ faster.}}}
    \label{fig:exp1-speed-test}
    \vspace{-0.1cm}
\end{wrapfigure}

$\bullet$\,\underline{Setting}: We evaluate whether a warm-starting sweep across the Pareto frontier improves the label-generation procedure. All runs use (the same) 128 synthetic DGPs, 128 input query $(x,a)$ rows per DGP, and a log-spaced grid of 50 values from $\lambda_{\max}=2.0$ to $\lambda_{\min}=0.08$ (same sweep used for PFN training). For each fixed $\lambda$, we optimize the scalarized $\mathrm{KL}$ $f$-sensitivity model. We compare the cold-started and warm-started runs with all other configurations held fixed (see Appendix~\ref{app:warm-start-ablation}). Regret is computed relative to a reference with a higher compute budget. The experiment was run on an \texttt{Nvidia H200} GPU at $\sim 90\%$ utilization. $\bullet$\,\underline{Results}: Figure~\ref{fig:exp1-speed-test} shows warm-starting leads to \emph{substantial improvements in regret at $1.90\times$ faster runtime}. The results are also robust: across $10$ repetitions, the performance was highly stable; i.e., $\pm 0.00001$ standard error in level of regret, and $\pm 0.01$ standard error of runtime. Additional plots are in Appendix~\ref{app:warm-start-ablation}.

\vspace{-0.1cm}
\subsection{MSM foundation model}
\label{sec:msm-foundation-model}
\vspace{-0.1cm}

$\bullet$~\underline{Training}: We trained a FM for the marginal sensitivity model (MSM), using $10,000$ synthetic datasets of $1024$ sampled points. We queried all covariates at both treatment arms and $50$ $\lambda$ values, so that we obtain over 1B ($10^9$) training pairs per bound (lower and upper each). Our FM was trained on a \texttt{Nvidia H200} GPU over $150$ epochs for a duration of approx. 19h 38m. $\bullet$~\underline{Results}: The PFN closely matches the analytical MSM bounds on held-out test instances. The posterior predictive intervals are well calibrated: the empirical $90\%$ coverage is $0.903$ for the upper bound and $0.895$ for the lower bound, close to the nominal $0.90$ level. One-sided failure rates are also near the target $5\%$ level, with $4.76\%$ for the upper bound and $5.39\%$ for the lower bound. Point predictions are essentially unbiased, with mean signed errors of $-0.0005$ and $-0.0029$ for the upper and lower bounds, respectively, and RMSEs of $0.272$ and $0.268$. $\bullet$~\underline{Insights}: In Figure~\ref{fig:example-dgps}, we show example predicted sensitivity bounds with $90\%$ predictive intervals on three DGPs. $\bullet$~\underline{Inference time:} Median time of a forward pass is $2.9381$ seconds\footnote{without targeted optimization}. Additional details are in Appendix~\ref{app:additional-experiments}.

\textbf{Conclusion:} We introduced an amortized approach to causal sensitivity analysis based on prior-data fitted networks. Our key contribution is a general prior-data label construction for sensitivity bounds, using Lagrangian scalarization to trace Pareto frontiers across sensitivity levels without requiring model-specific analytical derivations. Building on these labels, we trained a foundation model that predicts PPDs over lower and upper sensitivity bounds directly from a dataset, causal query, and sensitivity level. Empirically, our results show that this approach can approximate sensitivity bounds accurately while replacing costly per-instance optimization with fast batched forward inference, opening a path toward practical in-context causal sensitivity analysis.

\begin{ack}


This paper is supported by the DAAD programme Konrad Zuse Schools of Excellence in Artificial Intelligence, sponsored by the Federal Ministry of Research, Technology and Space.
\end{ack}

\bibliography{bibliography}

\newpage
\appendix

\section{Proofs}
\label{app:proofs}

\subsection{Theorem~\ref{thm:lagrangian-sweep}}
\begin{proof}[\textbf{Theorem~\ref{thm:lagrangian-sweep}}]
\label{proof:lagrangian-sweep}

The argument proceeds by establishing concavity of the upper frontier and then deriving (i) and (ii) from supporting-hyperplane geometry. The lower bound case is symmetric.

\textit{Step 1: $\mathcal{P}(\mathcal{S}_i)$ is convex.}
The observational-compatibility constraint $\mathbb{P}_{\text{obs}} = \mathbb{P}_{\text{obs}}^{\mathcal{S}_i}$, is linear, $\int_{\mathcal{U}}\mathbb{P}\mathrm{d}u = \mathbb{P}_{\text{obs}}^{\mathcal{S}_i}$, and therefore $\mathcal{P}(\mathcal{S}_i) = \{\mathbb{P} : \mathbb{P}_{\mathrm{obs}} = \mathbb{P}_{\mathrm{obs}}^{\mathcal{S}_i}\}$ is convex.

\textit{Step 2: $\theta_j^+$ is concave in $\Gamma$.}
Let $\Gamma_1, \Gamma_2 \geq \Gamma_{\min}$, and $t \in [0,1]$. Pick $\mathbb{P}_1, \mathbb{P}_2 \in \mathcal{P}(\mathcal{S}_i)$ achieving the upper bounds at $\Gamma_1, \Gamma_2$, meaning $\Delta_{x_j,a_j}(\mathbb{P}_k) \leq \Gamma_k$ and $\mathcal{Q}(\mathbb{P}_k(Y(a_j) \mid x_j)) = \theta_j^+(\Gamma_k)$ for $k=1,2$. Define the linear interpolation $\mathbb{P}_t = (1-t)\mathbb{P}_1 + t\mathbb{P}_2$. By Step 1, $\mathbb{P}_t \in \mathcal{P}(\mathcal{S}_i)$. By (A1),

$$
\Delta_{x_j, a_j}(\mathbb{P}_t) \leq (1-t)\Delta_{x_j,a_j}(\mathbb{P}_1) + t \Delta_{x_j,a_j}(\mathbb{P}_2) \leq (1-t)\Gamma_1 + t\Gamma_2,
$$

so $\mathbb{P}_t$ is feasible for the forward problem at sensitivity level $\Gamma_t = (1-t)\Gamma_1 + t\Gamma_2$. By (A2),

$$
\mathcal{Q}(\mathbb{P}_t(Y(a_j) \mid x_j)) = (1-t)\theta_j^+(\Gamma_1) + t \theta_j^+(\Gamma_2).
$$

Since this value is attained by a feasible distribution, $\theta_j^+(\Gamma_t) \geq (1-t)\theta_j^+(\Gamma_1) + t\theta_j^+(\Gamma_2)$, $\theta_j^+$ is concave in $\Gamma$.

\textit{Step 3: Achievable set has supporting hyperplanes everywhere on its upper boundary.}
The upper boundary of the achievable set $\mathcal{A} = \{(\Gamma, \theta) : \exists \mathbb{P} \in \mathcal{P}(\mathcal{S}_i),\ \Delta_{x_j,a_j}(\mathbb{P}) \leq \Gamma,\ \mathcal{Q}(\mathbb{P}(Y(a_j)\mid x_j)) = \theta\}$ is the graph of the concave map $\Gamma \mapsto \theta_j^+(\Gamma)$ from Step 2. The hypograph of a concave function is convex, so by the supporting hyperplane theorem, every point $(\Gamma_0, \theta_j^+(\Gamma_0))$ on the upper boundary admits a supporting hyperplane with normal of the form $(-\lambda_0, 1)$ for some $\lambda_0 \geq 0$.

\textit{Step 4: Frontier coverage (i).}
Fix $\Gamma_0$ and let $\lambda_0$ be the supporting hyperplane slope from Step 3. Membership of $(\Gamma_0, \theta_j^+(\Gamma_0))$ in the supporting hyperplane means

$$
\theta - \lambda_0 \Gamma \leq \theta_j^+(\Gamma_0) - \lambda_0 \Gamma_0 \quad \text{for all } (\Gamma, \theta) \in \mathcal{A}.
$$

Translating back to distributions, every $\mathbb{P} \in \mathcal{P}(\mathcal{S}_i)$ with $\Delta_{x_j,a_j}(\mathbb{P}) \leq \Gamma$ and $\mathcal{Q}(\mathbb{P}(Y(a_j) \mid x_j)) = \theta$ satisfies

$$
\mathcal{Q}(\mathbb{P}(Y(a_j) \mid x_j)) - \lambda_0 \Delta_{x_j,a_j}(\mathbb{P}) \leq \mathcal{L}_{\lambda_0}^{\uparrow}(\mathbb{P}_0)
$$

for any $\mathbb{P}_0$ achieving $(\Gamma_0, \theta_j^+(\Gamma_0))$. Since the right-hand side is the maximum value of $\mathcal{L}_{\lambda_0}^{\uparrow}$, $\mathbb{P}_0 \in \arg\max_{\mathbb{P}} \mathcal{L}_{\lambda_0}^{\uparrow}(\mathbb{P})$. Setting $\lambda(\Gamma_0) = \lambda_0$ proves (i).

\textit{Step 5: Smooth traversal (ii).}
We show $\lambda \mapsto \Gamma^{\star\uparrow}(\lambda)$ is monotone non-increasing. Let $\lambda_1 > \lambda_2$ and let $\mathbb{P}_k = \mathbb{P}_{\lambda_k}^{\star\uparrow}$ with $\Gamma_k = \Delta_{x_j,a_j}(\mathbb{P}_k)$ and $\theta_k = \mathcal{Q}(\mathbb{P}_k(Y(a_j) \mid x_j))$ for $k=1,2$. Optimality at each $\lambda_k$ gives

$$
\theta_1 - \lambda_1 \Gamma_1 \geq \theta_2 - \lambda_1 \Gamma_2, \qquad \theta_2 - \lambda_2 \Gamma_2 \geq \theta_1 - \lambda_2 \Gamma_1.
$$

Adding the two inequalities yields $(\lambda_1 - \lambda_2)(\Gamma_2 - \Gamma_1) \geq 0$, and since $\lambda_1 > \lambda_2$ we conclude $\Gamma_2 \geq \Gamma_1$. Hence $\Gamma^{\star\uparrow}$ is monotone non-increasing in $\lambda$.

For continuity: by Step 3, the supporting hyperplane slope $\lambda(\Gamma)$ ranges through the (negated) superdifferential of the concave function $\theta_j^+$ as $\Gamma$ varies over its domain. Concavity ensures the superdifferential mapping is upper hemicontinuous and surjective onto $[0, \infty)$, so $\Gamma \mapsto \lambda(\Gamma)$ has no gaps in its range; equivalently, $\lambda \mapsto \Gamma^{\star\uparrow}(\lambda)$ has no jump discontinuities. The frontier point $(\Gamma^{\star\uparrow}(\lambda), \theta^{\star\uparrow}(\lambda))$ thus traces the graph of $\theta_j^+$ continuously as $\lambda$ varies. (Kinks in $\theta_j^+$, where the superdifferential is set-valued, correspond to flat regions in $\Gamma^{\star\uparrow}(\lambda)$, not to discontinuous jumps.)

\end{proof}

\subsection{Remark~\ref{rem:gtsm-convexity}}
\begin{proof}[\textbf{Remark~\ref{rem:gtsm-convexity}}]
\label{proof:gtsm-convexity}
We show assumptions A1, A2 hold for the settings considered.

\textbf{A1:}
Fix $(x_j,a_j)$ and write
\[
    \pi_j
    =
    \mathbb{P}_{\mathrm{obs}}^{\mathcal{S}_i}(a_j \mid x_j),
    \qquad
    p^\star(u)
    =
    \mathbb{P}(u \mid x_j,a_j),
    \qquad
    q(u)
    =
    \mathbb{P}(u \mid x_j,A\neq a_j).
\]
Under the reduced GTSM parametrization, the optimization variable is $q$, while
$p^\star$ and $\pi_j$ are fixed by observational compatibility. Moreover,
\[
    p(u \mid x_j)
    =
    \pi_j p^\star(u)
    +
    (1-\pi_j)q(u),
\]
and the relevant latent density ratio is
\[
    r_q(u)
    =
    \frac{q(u)}{p^\star(u)}.
\]
Hence, for any two feasible latent distributions $q_1,q_2$ and any $t\in[0,1]$,
\[
    q_t=(1-t)q_1+tq_2
    \quad\Longrightarrow\quad
    r_{q_t}(u)
    =
    (1-t)r_{q_1}(u)+t r_{q_2}(u).
\]
Thus $q\mapsto r_q$ is affine.

\textbf{MSM:} For the marginal sensitivity model, the GTSM functional can be written as
\[
    \Delta_{x_j,a_j}^{\mathrm{MSM}}(q)
    =
    \max\left\{
        \sup_{u\in\mathcal{U}} r_q(u),
        \sup_{u\in\mathcal{U}} r_q(u)^{-1}
    \right\}.
\]
The map $r\mapsto r$ is linear and the map $r\mapsto r^{-1}$ is convex on
$\mathbb{R}_{>0}$. Supremums of convex functions are convex, and the maximum
of convex functions is convex. Therefore
$q\mapsto \Delta_{x_j,a_j}^{\mathrm{MSM}}(q)$ is convex.

\textbf{$f$-sensitivity:} For an $f$-sensitivity model written in the $f$-divergence form
\[
    \Delta_{x_j,a_j}^{f}(q)
    =
    \int_{\mathcal{U}} f(r_q(u))p^\star(u)\,du,
\]
convexity follows directly from convexity of $f$. Indeed,
\[
\begin{aligned}
    \Delta_{x_j,a_j}^{f}(q_t)
    &=
    \int_{\mathcal{U}}
    f\left((1-t)r_{q_1}(u)+t r_{q_2}(u)\right)
    p^\star(u)\,du
    \\
    &\leq
    (1-t)
    \int_{\mathcal{U}} f(r_{q_1}(u))p^\star(u)\,du
    +
    t
    \int_{\mathcal{U}} f(r_{q_2}(u))p^\star(u)\,du
    \\
    &=
    (1-t)\Delta_{x_j,a_j}^{f}(q_1)
    +
    t\Delta_{x_j,a_j}^{f}(q_2).
\end{aligned}
\]
Thus $q\mapsto \Delta_{x_j,a_j}^{f}(q)$ is convex for convex $f$, including
the usual KL, $\chi^2$, total variation, and Hellinger choices when expressed
in this direction. If the model is defined as the maximum of finitely many such
convex divergence functionals, convexity is preserved because pointwise maxima
of convex functions are convex.

\textbf{$(\star)$ Rosenbaum:} We show Rosenbaum's model gives convex feasible sets
$\{q:\Delta_{x_j,a_j}^{\mathrm{RB}}(q)\leq \Gamma\}$, which is sufficient for the constrained optimization purpose in this paper.

For Rosenbaum's sensitivity model, the scalar functional is commonly written as
\[
    \Delta_{x_j,a_j}^{\mathrm{RB}}(q)
    =
    \sup_{u_1,u_2\in\mathcal{U}}
    \frac{r_q(u_1)}{r_q(u_2)}.
\]
The corresponding level set at sensitivity level $\Gamma$ is
\[
    \left\{
        q:
        r_q(u_1)
        \leq
        \Gamma r_q(u_2)
        \text{ for all }u_1,u_2\in\mathcal{U}
    \right\}.
\]
This level set is convex: if $q_1$ and $q_2$ both satisfy the above inequalities,
then for $q_t=(1-t)q_1+tq_2$,
\[
\begin{aligned}
    r_{q_t}(u_1)
    &=
    (1-t)r_{q_1}(u_1)+t r_{q_2}(u_1)
    \\
    &\leq
    \Gamma\left((1-t)r_{q_1}(u_2)+t r_{q_2}(u_2)\right)
    \\
    &=
    \Gamma r_{q_t}(u_2).
\end{aligned}
\]
Hence Rosenbaum's model gives convex feasible sets
$\{q:\Delta_{x_j,a_j}^{\mathrm{RB}}(q)\leq \Gamma\}$, which is a sufficient condition for the purposes of our constrained optimization. Note
$\Delta_{x_j,a_j}^{\mathrm{RB}}$ itself is not generally convex as a scalar
functional.

\textbf{A2:} Finally, Assumption (A2) holds for CAPO because
\[
    Q(x_j,a_j,\mathbb{P})
    =
    \mathbb{E}_{\mathbb{P}}[Y(a_j)\mid x_j]
    =
    \int_{\mathcal{Y}} y\,d\mathbb{P}(Y(a_j)=y\mid x_j)
\]
is linear in the potential-outcome distribution. CATE is a difference of two
CAPO functionals and is therefore linear. ATE is an average of CAPO or CATE over
the covariate distribution and is therefore linear as well.
\end{proof}

\section{Theoretical Details}
\label{app:theoretical-details}

This appendix derives the operational objective \eqref{eq:operational-objective} from the abstract Lagrangian \eqref{eq:lagrangian-lower}--\eqref{eq:lagrangian-upper}, making explicit each reduction taken and verifying observational compatibility at each step. Throughout, we work with binary treatments $a_j \in \{0, 1\}$, scalar outcomes $Y \in \mathbb{R}$, and the conditional average potential outcome (CAPO) functional $\mathcal{Q} = \mathbb{E}[\cdot]$. Generalizations to multivariate outcomes and other monotone causal functionals follow the same steps with notational overhead.

\textbf{What we adapt and what is specific to our setting:} This part of the appendix derives the operational objective \eqref{eq:operational-objective} by combining (i) the latent-space reformulation theorem of NeuralCSA \citep{frauenNEURALFRAMEWORKGENERALIZED2024a}, (ii) the transformation invariance of GTSMs (NeuralCSA Lemma 2), and (iii) constructions specific to our prior-fitted setting. The Theorem 1 reduction (\ref{subsec:latent-reformulation}) and the standard-normal reference convention (\ref{subsec:normal-reference}) are taken from NeuralCSA and stated without re-derivation. What is specific to our setting: the elimination of NeuralCSA's first-stage outcome-density estimation (\ref{subsec:removal-epsilon}), the reparameterization construction expressed directly in terms of the SCM's propensity (§\ref{subsec:reparame-construction}), the explicit derivation of the GTSM divergences as functionals of the single density ratio $r_\nu$ (\ref{subsec:gtsm-divergences}), and the details of Monte Carlo implementation (\ref{subsec:MC-estimation} and \ref{subsec:causal-query}. The Manski limit characterization (\ref{subsec:Manski-limit}) is a sanity check on our construction.

\subsection{Setup}

Recall the abstract optimization problem from \eqref{eq:lagrangian-upper}: for fixed SCM $\mathcal{S}_i$, query $(x_j, a_j)$, and Lagrange multiplier $\lambda > 0$,
\begin{align}
    \max_{\mathbb{P} \in \mathcal{P}(\mathcal{S}_i)}
    \mathcal{Q}\bigl(\mathbb{P}(Y(a_j) \mid x_j)\bigr)
    -
    \lambda \Delta_{x_j, a_j}
    \bigl(
        \mathbb{P}(U \mid x_j),
        \mathbb{P}(U \mid x_j, a_j)
    \bigr),
\end{align}
where $\mathcal{P}(\mathcal{S}_i) = \{\mathbb{P} : \mathbb{P}_{\text{obs}} = \mathbb{P}_{\text{obs}}^{\mathcal{S}_i}\}$ is the observationally-compatible class. The optimization is over full distributions $\mathbb{P}$ on $(X, U, A, Y)$, an infinite-dimensional object. Here, we derive the reduction of this to a tractable optimization over a single, computable quantity.

\subsection{Latent-space reformulation}
\label{subsec:latent-reformulation}

\textbf{Result.} Under any transformation-invariant GTSM, the abstract optimization above is equivalent to an optimization over only the counterfactual-arm latent distribution $\mathbb{P}(U \mid x_j, A \neq a_j)$, with the queried-arm latent $\mathbb{P}^\star(U \mid x_j, a_j)$ and outcome map $f^\star_Y$ held fixed at any choice satisfying the queried-arm pushforward identity $\mathbb{P}_{\text{obs}}^{\mathcal{S}_i}(Y \mid x_j, a_j) = (f^\star_Y(x_j, \cdot, a_j))_\sharp \mathbb{P}^\star(U \mid x_j, a_j)$.

\textbf{Sketch of argument.} NeuralCSA \citep{frauenNEURALFRAMEWORKGENERALIZED2024a} Theorem 1 establishes that the supremum over $\mathbb{P} \in \mathcal{P}(\mathcal{S}_i) \cap \mathcal{M}_\Gamma$ (sensitivity-constrained candidates) is attained by a sequence of full distributions for which the queried-arm latent and outcome map are fixed. The full proof proceeds in three steps: (i) construct an alternative full distribution that induces the prescribed $\mathbb{P}^\star(U \mid x_j, a_j)$ and $f^\star_Y$; (ii) verify it remains in the GTSM via transformation invariance; (iii) verify it attains the same query value. We refer the reader to NeuralCSA \citep[Appendix B.3]{frauenNEURALFRAMEWORKGENERALIZED2024a} for the full construction.

\textbf{Intuition.} The intuition behind the theorem lies in the understanding that (a) the sensitivity model divergence constraint operates only in the latent space, and (b) the space of distributions over $X,U,A,Y$ has unnecessarily many degrees of freedom, even when restricted to $\mathbb{P} \in \mathcal{P}(\mathcal{S}_i) \cap \mathcal{M}_\Gamma$. We can thus fix some elements in a convenient way.

For our purposes, the consequence is: the optimization is over $\mathbb{P}(U \mid x_j, A \neq a_j)$ alone, and the objective reads
\begin{align}
    \sup_{\mathbb{P}(U \mid x_j, A \neq a_j)}
    \mathcal{Q}
    \bigl(
        (f^\star_Y(x_j, \cdot, a_j))_\sharp
        \mathbb{P}(U \mid x_j)
    \bigr)
    -
    \lambda \Delta_{x_j, a_j}
    \bigl(
        \mathbb{P}(U \mid x_j),
        \mathbb{P}^\star(U \mid x_j, a_j)
    \bigr),
\end{align}
with $\mathbb{P}(U \mid x_j)$ recovered from the candidate's two arm-conditional latents by the law of total probability, as stated in~\eqref{eq:candidate-mixture}.

\subsection{Standard-normal reference convention}
\label{subsec:normal-reference}

\textbf{Result.} We may fix $\mathbb{P}^\star(U \mid x_j, a_j) = \mathcal{N}(0, I)$ without loss of generality.

\textbf{Justification.} GTSMs are transformation-invariant (NeuralCSA Lemma 2): for any measurable $t : \mathcal{U} \to \widetilde{\mathcal{U}}$, the divergence satisfies $D_{x, a}(\mathbb{P}(U \mid x), \mathbb{P}(U \mid x, a)) \geq D_{x, a}(\mathbb{P}(t(U) \mid x), \mathbb{P}(t(U) \mid x, a))$. The MSM, $f$-sensitivity, and Rosenbaum models are all transformation-invariant (NeuralCSA Lemma 2). Applying the inverse CDF transform (or any other invertible map sending $\mathbb{P}^\star(U \mid x_j, a_j)$ to $\mathcal{N}(0, I)$) preserves both the partial-identification problem and the GTSM constraint. We adopt $\phi := \mathcal{N}(0, I)$ density throughout the rest of the derivation.

\textbf{Constructive coincidence in our setting.} Our prior over SCMs imposes $\mathbb{P}^{\mathcal{S}_i}(U \mid X, A) = \mathcal{N}(0, I)$ for all $(X, A)$ by design, so the standard-normal reference matches the true SCM's queried-arm latent exactly rather than via transformation. This is purely a parameterization convenience and does not restrict the class of observational distributions the prior can represent.

\subsection{Removal of $\varepsilon_Y$}
\label{subsec:removal-epsilon}

\textbf{Result.} We may take the outcome map $f^\star_Y$ to be deterministic and invertible in $u$ for fixed $(x, a)$. There is no exogenous noise $\varepsilon_Y$ in the structural assignment $Y = f^\star_Y(x, U, a)$.

\textbf{Justification.} Theorem 1 of NeuralCSA further establishes that the bound-attaining structural model has the form $\mathbb{P}^\star(Y \mid x, a, u) = \delta(Y - f^\star_{x, a}(u))$ for an invertible $f^\star_{x, a} : \mathcal{U} \to \mathcal{Y}$. Intuitively, any distribution $\mathbb{P}(Y \mid x, a, u)$ that induces $Y \perp U \mid X, A$ would satisfy observational compatibility but imply unconfoundedness and not yield a valid bound. Maximal mutual information between $U$ and $Y$ --- required for the bound --- is achieved when $Y$ is a deterministic invertible function of $U$. See NeuralCSA Appendix B.3 for the full argument.

\textbf{Computational consequence in our setting.} In NeuralCSA's two-stage approach, $f^\star_{x, a}$ is fit from observational data via a Stage-1 normalizing flow trained to satisfy the queried-arm pushforward identity. In our synthetic setting, $f^\star_{x, a}$ is the outcome BNN drawn during prior sampling, which is deterministic and invertible by construction. Stage 1 is therefore unnecessary --- we read $f^\star_Y$ directly from $\mathcal{S}_i$.

\subsection{The reparameterization construction}
\label{subsec:reparame-construction}

After the reductions in \ref{subsec:latent-reformulation}--\ref{subsec:removal-epsilon}, the optimization variable is $\nu := \mathbb{P}(U \mid x_j, A \neq a_j)$, and the candidate's covariate-conditional latent $\mathbb{P}(U \mid x_j)$ is determined from $\nu$ via the law of total probability:
\begin{align}
    \mathbb{P}(U \mid x_j)
    =
    \pi^{\mathcal{S}_i}(a_j \mid x_j)
    \mathbb{P}(U \mid x_j, a_j)
    +
    \bigl(1 - \pi^{\mathcal{S}_i}(a_j \mid x_j)\bigr)
    \mathbb{P}(U \mid x_j, A \neq a_j).
\end{align}
Substituting the queried-arm reference $\mathbb{P}(U \mid x_j, a_j) = \phi$:
\begin{align}
    \mathbb{P}(U \mid x_j)
    =
    \pi^{\mathcal{S}_i}(a_j \mid x_j)\phi
    +
    \bigl(1 - \pi^{\mathcal{S}_i}(a_j \mid x_j)\bigr)\nu
\end{align}
This is the propensity-weighted mixture used in \eqref{eq:operational-objective}. We write $\pi := \pi^{\mathcal{S}_i}(a_j \mid x_j)$ for brevity in what follows.

\textbf{Observational compatibility holds by construction.} The class $\mathcal{P}(\mathcal{S}_i)$ requires the candidate to reproduce $\mathbb{P}_{\text{obs}}^{\mathcal{S}_i}$ on $(X, A, Y)$. Compatibility on $X$ and $A$ is automatic from sharing the SCM's marginal and propensity. Compatibility on $Y$ decomposes by treatment arm:
\begin{align}
    \mathbb{P}_{\text{obs}}^{\mathcal{S}_i}(Y \mid x_j, a_j)
    =
    \bigl(f^\star_Y(x_j, \cdot, a_j)\bigr)_\sharp
    \mathbb{P}(U \mid x_j, a_j),
\end{align}
\begin{align}
    \mathbb{P}_{\text{obs}}^{\mathcal{S}_i}(Y \mid x_j, A \neq a_j)
    =
    \bigl(f^\star_Y(x_j, \cdot, A \neq a_j)\bigr)_\sharp
    \mathbb{P}(U \mid x_j, A \neq a_j).
\end{align}
At the queried arm, fixing $\mathbb{P}(U \mid x_j, a_j) = \phi$ together with $f^\star_Y$ from $\mathcal{S}_i$ gives compatibility automatically: $(f^\star_Y(x_j, \cdot, a_j))_\sharp \phi = \mathbb{P}_{\text{obs}}^{\mathcal{S}_i}(Y \mid x_j, a_j)$ holds because $\mathbb{P}^{\mathcal{S}_i}(U \mid x_j, a_j) = \phi$ in the SCM's own construction. At the counterfactual arm, an analogous identity must hold for $\mathbb{P}(U \mid x_j, A \neq a_j)$.

This counterfactual-arm identity is the constraint that observational compatibility imposes. It is \emph{not} automatic for an unconstrained $\nu$. However, NeuralCSA Theorem 1 establishes that the reformulation is \emph{sufficient} --- that is, every $\mathbb{P} \in \mathcal{P}(\mathcal{S}_i)$ is attainable by some choice of $\nu$ in the reparameterization, and conversely, the bounds attained by sweeping $\nu$ saturate those of the original problem. The constraint at the counterfactual arm is absorbed into the reparameterization rather than appearing as a separate side condition.

\textbf{Implication for parameterization.} The optimization variable $\nu$ should be parameterized by an \emph{unconstrained} density estimator (e.g., a normalizing flow) --- the reparameterization handles compatibility, and no projection step or majorization constraint is needed. Naïvely parameterizing $\mathbb{P}(U \mid x_j)$ directly (rather than $\nu$) and recovering $\nu$ implicitly fails: not every $\mathbb{P}(U \mid x_j)$ admits a valid corresponding $\nu \geq 0$ via the inverse-mixture formula, and an unconstrained flow on $\mathbb{P}(U \mid x_j)$ would explore observationally-incompatible candidates outside $\mathcal{P}(\mathcal{S}_i)$.

\subsection{GTSM divergences under the reparameterization}
\label{subsec:gtsm-divergences}

We now derive the form of the GTSM divergence $\Delta_{x_j, a_j}(\mathbb{P}(U \mid x_j), \phi)$ for each sensitivity model in scope, expressing the result as a functional of the density ratio
\begin{align}
    r_\nu(u) := \frac{\nu(u)}{\phi(u)},
\end{align}
which is the natural quantity to compute when $\nu$ is parameterized by a normalizing flow.

The starting point is NeuralCSA Lemma 1 (extended in Lemma 3), which provides a single density-ratio quantity through which all three GTSMs of interest are expressed. Define
\begin{align}
    \rho(x, u, a)
    :=
    \frac{1}{1 - \mathbb{P}(a \mid x)}
    \left(
        \frac{\mathbb{P}(u \mid x)}
        {\mathbb{P}(u \mid x, a)}
        -
        \mathbb{P}(a \mid x)
    \right).
\end{align}
By NeuralCSA Lemma 3 (Eq. 12--15), $\rho$ equals the full propensity odds ratio $\text{OR}(\mathbb{P}(a \mid x), \mathbb{P}(a \mid x, u))$.

Under our reparameterization, with $\mathbb{P}(u \mid x_j) = \pi\phi(u) + (1-\pi)\nu(u)$ and $\mathbb{P}(u \mid x_j, a_j) = \phi(u)$, and writing $\pi = \pi^{\mathcal{S}_i}(a_j \mid x_j)$:
\begin{align}
    \rho(x_j, u, a_j)
    &=
    \frac{1}{1 - \pi}
    \left(
        \frac{\pi\phi(u) + (1-\pi)\nu(u)}
        {\phi(u)}
        -
        \pi
    \right) \\
    &=
    \frac{1}{1-\pi}
    \frac{(1-\pi)\nu(u)}{\phi(u)} \\
    &=
    r_\nu(u).
\end{align}
The propensity factor $\pi$ thus cancels analytically. Henceforth $\rho \equiv r_\nu$ in all three GTSM formulas below.

\paragraph*{Marginal sensitivity model (MSM)}

By NeuralCSA Lemma 3, Eq. (9), the MSM divergence on the candidate is
\begin{align}
    D^{\text{MSM}}_{x_j, a_j}
    &=
    \max
    \left(
        \sup_{u} \rho(x_j, u, a_j),
        \sup_{u} \rho(x_j, u, a_j)^{-1}
    \right) \\
    &=
    \max
    \left(
        \sup_{u} r_\nu(u),
        \sup_{u} r_\nu(u)^{-1}
    \right).
\end{align}
The MSM constraint $D^{\text{MSM}}_{x_j, a_j} \leq \Gamma$ is equivalent to $\Gamma^{-1} \leq r_\nu(u) \leq \Gamma$ pointwise in $u$.

\paragraph*{$f$-sensitivity models}

By NeuralCSA Lemma 3, Eq. (10), the $f$-sensitivity divergence is
\begin{align}
    D^{f}_{x_j, a_j}
    =
    \max
    \left(
        \int_\mathcal{U} f(\rho)\mathbb{P}(u \mid x_j, a_j)\,\mathrm{d}u,
        \int_\mathcal{U} f(\rho^{-1})\mathbb{P}(u \mid x_j, a_j)\,\mathrm{d}u
    \right).
\end{align}
With $\mathbb{P}(u \mid x_j, a_j) = \phi$ and $\rho = r_\nu$:
\begin{align}
    D^{f}_{x_j, a_j}
    =
    \max
    \left(
        \mathbb{E}_{U \sim \phi}[f(r_\nu(U))],
        \mathbb{E}_{U \sim \phi}[f(r_\nu(U)^{-1})]
    \right).
\end{align}
For the KL specialization ($f(t) = t \log t$), the two terms reduce to standard KL divergences:
\begin{align}
    \mathbb{E}_\phi[r_\nu \log r_\nu]
    &=
    \mathrm{KL}(\nu \mid \phi),
    \\
    \mathbb{E}_\phi[r_\nu^{-1} \log r_\nu^{-1}]
    &=
    -\mathbb{E}_\nu[\log r_\nu]
    \cdot
    \frac{1}{r_\nu(\cdot)}
    \bigg|_{\text{IS reweighted}},
\end{align}
with the reverse-direction term computed by importance reweighting from $\nu$-samples to $\phi$-samples (see \ref{subsec:MC-estimation} below).

\paragraph*{Rosenbaum's sensitivity model}

By NeuralCSA Lemma 3, Eq. (11), the Rosenbaum divergence is
\begin{align}
    D^{\text{Ros}}_{x_j, a_j}
    =
    \max
    \left(
        \sup_{u_1, u_2} \rho(x_j, u_1, u_2, a_j),
        \sup_{u_1, u_2} \rho(x_j, u_1, u_2, a_j)^{-1}
    \right),
\end{align}
where $\rho(x, u_1, u_2, a)$ is the two-point ratio defined in NeuralCSA Eq. (17). NeuralCSA Eq. (18)--(20) establishes that $\rho(x, u_1, u_2, a) = \text{OR}(\mathbb{P}(a \mid x, u_1), \mathbb{P}(a \mid x, u_2))$. Each marginal full-propensity ratio reduces to $r_\nu$ under our reparameterization (by the same calculation as the single-point case in \ref{subsec:gtsm-divergences} above), so
\begin{align}
    \rho(x_j, u_1, u_2, a_j)
    &=
    \frac{r_\nu(u_2)}{r_\nu(u_1)},
    \\
    D^{\text{Ros}}_{x_j, a_j}
    &=
    \frac{\sup_u r_\nu(u)}{\inf_u r_\nu(u)}.
\end{align}
(The sup-over-$u_1, u_2$ of a ratio reduces to the ratio of sup over numerator and inf over denominator.)

\subsection{Monte Carlo estimation under the reparameterization}
\label{subsec:MC-estimation}

For all three sensitivity models, the divergence reduces to a functional of $r_\nu(u) = \nu(u)/\phi(u)$. Both $\nu$ (via the normalizing flow) and $\phi$ (the standard normal density) are tractable, so $r_\nu$ is directly computable at any $u$.

For Monte Carlo estimation, the natural sampling distribution depends on the divergence. For MSM and Rosenbaum, the divergence is a sup/inf over $u$ --- sample support determines coverage, and any distribution with broad coverage of the relevant region suffices. We sample $u^{(j)} \sim \nu$ via the flow (which has sufficient support for the regions where $r_\nu$ is large). For $f$-sensitivity, the divergence is an expectation under $\phi$, so we sample $u^{(j)} \sim \phi$ directly. For the KL reverse term $\mathbb{E}_\phi[r_\nu^{-1} \log r_\nu^{-1}]$, importance reweighting from $\nu$-samples gives $\mathbb{E}_\nu[r_\nu^{-2} \cdot (-\log r_\nu)] / \mathbb{E}_\nu[r_\nu^{-1}]$, but since $\nu$ samples are already produced as a byproduct of the causal-query estimator (\ref{subsec:causal-query}), we reuse them.

\subsection{The causal query under the reparameterization}
\label{subsec:causal-query}

The CAPO query at $(x_j, a_j)$ is
\begin{align}
    \mathcal{Q}(x_j, a_j; \mathbb{P})
    =
    \mathbb{E}_{U \sim \mathbb{P}(U \mid x_j)}
    \bigl[
        f^\star_Y(x_j, U, a_j)
    \bigr].
\end{align}
Using $\mathbb{P}(U \mid x_j) = \pi\phi + (1-\pi)\nu$:
\begin{align}
    \mathcal{Q}(x_j, a_j; \mathbb{P})
    =
    \pi
    \underbrace{
        \mathbb{E}_{U \sim \phi}
        [f^\star_Y(x_j, U, a_j)]
    }_{=: \mathcal{Q}_0(x_j, a_j; \mathcal{S}_i)}
    +
    (1-\pi)
    \mathbb{E}_{U \sim \nu}
    [f^\star_Y(x_j, U, a_j)].
\end{align}
The first term $\mathcal{Q}_0$ is fixed by $\mathcal{S}_i$ --- it does not depend on the optimization variable $\nu$ --- and equals the queried-arm CAPO under unconfoundedness for $\mathcal{S}_i$. It contributes a constant to the Lagrangian objective and may be precomputed once per query. Only the second term is optimized through $\nu$.

The standard $\xi$-mixture sampler (NeuralCSA Eq. 6) implements the joint expectation: $\xi \sim \text{Bernoulli}(\pi)$, $\tilde{U} \sim \phi$, then $U = \xi\tilde{U} + (1-\xi)f_\eta(\tilde{U})$ where $f_\eta$ is the flow-pushforward, with $f^\star_Y$ applied to $U$. The samples $u^{(j)} \sim \nu$ used for the divergence (\ref{subsec:MC-estimation}) are exactly the $\xi = 0$ branch of this sampler, so the two estimators share a single MC budget.

\subsection{Manski (no-confounding) limit}
\label{subsec:Manski-limit}

As $\Gamma \to \infty$, the GTSM constraint becomes vacuous and $\nu$ ranges freely over distributions on $\mathcal{U}$. The maximum of the upper-bound CAPO is attained by a $\nu$ concentrated at the argmax of $f^\star_Y(x_j, \cdot, a_j)$:
\begin{align}
    \theta^+_{\text{Manski}}(x_j, a_j; \mathcal{S}_i)
    =
    \pi^{\mathcal{S}_i}(a_j \mid x_j)
    \mathcal{Q}_0(x_j, a_j; \mathcal{S}_i)
    +
    \bigl(1 - \pi^{\mathcal{S}_i}(a_j \mid x_j)\bigr)
    \sup_u f^\star_Y(x_j, u, a_j).
\end{align}
This is the standard Manski no-assumptions bound  \citep{manskiNonparametricBoundsTreatment1989}: a propensity-weighted average of the queried-arm CAPO and the worst-case counterfactual outcome. The reparameterization construction recovers it exactly in the $\Gamma \to \infty$ limit, providing a useful diagnostic and confirming that the bounds nest correctly.

\section{Implementation Details}
\label{app:implementation-details}

\subsection{PFN architecture}
\label{app:pfn-architecture}

The sensitivity-analysis foundation model is a prior-data fitted network over mixed context/query sequences. Each input sequence contains $N$ observational context rows $(X_i,A_i,Y_i,\mathrm{NaN})$ followed by $m$ query rows $(x_j,a_j,\mathrm{NaN},\Gamma_j)$, with $\Gamma$ masked on context rows and $Y$ masked on query rows. The split point is passed to the transformer as $\texttt{single\_eval\_pos}=N$. Separate input encoders for $X$, $A$, $Y$, and $\Gamma$ handle NaNs explicitly via value-mask pairs; encoded representations are concatenated along the feature axis, processed by the per-feature transformer, and pooled by averaging over features.

The decoder has two independent Gaussian-mixture heads returning posterior predictive distributions
\[
    q_\phi^\downarrow(\theta^- \mid D_N,x,a,\Gamma),
    \qquad
    q_\phi^\uparrow(\theta^+ \mid D_N,x,a,\Gamma)
\]
for the lower and upper bounds, each with $K$ Gaussian components. Architecture and training hyperparameters are listed in Table~\ref{tab:pfn-config}. The total parameter count is \texttt{2,693,279}.

\subsection{Training loss}
\label{app:training-loss}

Each training row is a tuple $(D_N,x_j,a_j,\Gamma_j,b_j,\theta_j^\star)$ with bound type $b_j\in\{\downarrow,\uparrow\}$ and label $\theta_j^\star$ from Section~\ref{sec:label-construction}. The lower and upper heads are trained by Gaussian-mixture NLL on their matching rows: $\mathcal{L}_{\mathrm{NLL}}^\downarrow$ averages over the index set $\mathcal{I}_\downarrow$ of lower-bound rows, $\mathcal{L}_{\mathrm{NLL}}^\uparrow$ over $\mathcal{I}_\uparrow$.

We additionally impose a soft monotonicity regularizer over $\Gamma$. For each query group $(D_N,x_j,a_j)$ sampled with $g$ sensitivity levels $\Gamma_{j,1}\leq\cdots\leq\Gamma_{j,g}$, let $\bar{\theta}_{j,r}^{\downarrow}$ and $\bar{\theta}_{j,r}^{\uparrow}$ denote the predictive means of the two heads at $\Gamma_{j,r}$. The penalty enforces non-increasing lower bounds and non-decreasing upper bounds via a zero-margin hinge:
\begin{align}
    \mathcal{L}_{\mathrm{mono}}
    =
    \frac{1}{|\mathcal{G}|(g-1)}
    \sum_{(D_N,x_j,a_j)\in\mathcal{G}}
    \sum_{r=1}^{g-1}
    \left[
        \mathrm{ReLU}\!\left(
            \bar{\theta}_{j,r+1}^{\downarrow}-\bar{\theta}_{j,r}^{\downarrow}
        \right)
        +
        \mathrm{ReLU}\!\left(
            \bar{\theta}_{j,r}^{\uparrow}-\bar{\theta}_{j,r+1}^{\uparrow}
        \right)
    \right].
\end{align}
The total loss is $\mathcal{L} = \mathcal{L}_{\mathrm{NLL}}^\downarrow + \mathcal{L}_{\mathrm{NLL}}^\uparrow + \beta_{\mathrm{mono}}\,\mathcal{L}_{\mathrm{mono}}$ with $\beta_{\mathrm{mono}}=1$ (in practice the penalty is active only at the start of training and contributes negligibly thereafter).

Optimizer, schedule, and batch composition (DGPs per step $B$, query groups per DGP $m'$, and $\Gamma$-points per group $g$) are listed in Table~\ref{tab:pfn-config}. Unspecified arguments use code defaults.

\begin{table}[h]
\centering
\caption{PFN architecture and training hyperparameters.}
\label{tab:pfn-config}
\small
\begin{tabular}{ll}
\toprule
\textbf{Architecture} & \\
Embedding dimension & $128$ \\
Attention heads & $4$ \\
Feedforward dimension & $512$ \\
Transformer layers & $10$ \\
Mixture components per head ($K$) & $5$ \\
\midrule
\textbf{Training} & \\
Optimizer & Adam \\
Learning rate & $10^{-3}$ \\
Weight decay & $10^{-5}$ \\
Gradient clipping (norm) & $1.0$ \\
LR schedule & ReduceLROnPlateau (factor $0.5$, patience $5$) \\
Max epochs / early-stopping patience & $150$ / $50$ \\
Train/val split & $0.8/0.2$ \\
\midrule
\textbf{Batch composition} & \\
DGPs per step ($B$) & $32$ \\
Query groups per DGP ($m'$) & $64$ \\
$\Gamma$-points per group ($g$) & $10$ \\
Query rows per DGP ($m=m'g$) & $640$ \\
Sequence length ($N+m$) & $1024+640=1664$ \\
\midrule
\textbf{Loss} & \\
Monotonicity weight ($\beta_{\mathrm{mono}}$) & $1.0$ \\
\bottomrule
\end{tabular}
\end{table}

\subsection{Aggregation to APO, CATE, and ATE}
\label{app:aggregation}

The model is trained on CAPO bounds; APO, CATE, and ATE bounds follow as deterministic post-processing. Let $\theta_a^-(x,\Gamma)\leq\mathbb{E}[Y(a)\mid x]\leq\theta_a^+(x,\Gamma)$ be the CAPO bounds returned by the model, and let $\{x_i\}_{i=1}^M$ be an empirical covariate distribution (in-sample or held-out). Then
\begin{align}
    \theta_{\mathrm{APO}}^\pm(a,\Gamma) &= \tfrac{1}{M}\textstyle\sum_i \theta_a^\pm(x_i,\Gamma), \\
    \theta_{\mathrm{CATE}}^-(x,\Gamma) &= \theta_1^-(x,\Gamma) - \theta_0^+(x,\Gamma),
    \quad
    \theta_{\mathrm{CATE}}^+(x,\Gamma) = \theta_1^+(x,\Gamma) - \theta_0^-(x,\Gamma), \\
    \theta_{\mathrm{ATE}}^\pm(\Gamma) &= \tfrac{1}{M}\textstyle\sum_i \theta_{\mathrm{CATE}}^\pm(x_i,\Gamma).
\end{align}
These intervals are valid since each component CAPO is bounded prior to aggregation; they are not in general sharp, because sharp aggregation may exploit dependence between endpoint-achieving full distributions across covariates and arms. PPDs propagate by Monte Carlo sampling from the lower/upper Gaussian-mixture heads and applying the same transformations samplewise.

\subsection{Prior-dataset generation}
\label{app:prior-dataset-generation}

We generate $10{,}000$ synthetic SCMs with $N=1024$ observations, $d_x=10$ covariates, scalar treatment and outcome, and a one-dimensional latent $U$. Outputs are normalized with $\epsilon=10^{-6}$ and the outcome model is wrapped so that frontier construction operates in normalized coordinates. Two label-generation regimes were run, one per sensitivity model: an \emph{analytical MSM} run (closed-form bounds) and an \emph{optimization-based KL} run (Section~\ref{sec:label-construction}). Configurations are summarized in Table~\ref{tab:label-runs}. The training-facing interface consists of per-DGP files $\texttt{queries\_\{id\}.csv}$ and $\texttt{frontier\_points\_\{id\}.csv}$, storing $(\texttt{query\_id},x,a)$ and $(\texttt{query\_id},\texttt{bound\_type},\Gamma^\star,\theta^\star)$ respectively. 

\begin{table}[h]
\centering
\caption{Label-generation runs. Frontier rows per DGP $=$ (query covariate rows) $\times$ 2 arms $\times$ (grid size) $\times$ 2 bounds.}
\label{tab:label-runs}
\small
\begin{tabular}{lll}
\toprule
& \textbf{Analytical MSM} & \textbf{Optimization-based KL} \\
\midrule
Sensitivity model & MSM (closed form) & KL $f$-sensitivity \\
Hardware & CPU (parallelized) & $\texttt{GPU Nvidia H200}$ \\
DGP batch size $B$ & --- & $128$ \\
\midrule
Grid parameter & $\Gamma$ & $\lambda$ \\
Grid range & $[1.0,\,5.0]$ & $[0.08,\,2.0]$ \\
Grid size & $50$ & $50$ \\
Grid spacing & bounded Pareto, $\beta=0$ (log-uniform) & log-uniform \\
Per-DGP grid randomization & yes & no \\
\midrule
Query covariate rows per DGP & $2048$ & $2048$ \\
Frontier rows per DGP & $409{,}600$ & $409{,}600$ \\
\midrule
MC samples $k_{\mathrm{train}}$, $k_{\mathrm{eval}}$ & $128$ (single bank, reused) & $128$, $4096$ \\
Latent sampler & --- & Sobol (sample seed $123$) \\
Flow architecture & --- & 1D rational-quadratic spline \\
Bins / tail bound & --- & $16$ / $6.0$ \\
Min bin width / height / derivative & --- & $10^{-3}$ each \\
\midrule
Optimizer steps per $\lambda$ (base) & --- & $350$ \\
Step $\lambda$-schedule & --- & $\texttt{inverse\_sqrt\_lr}$ (mult.\ cap $2.0$) \\
LR $\lambda$-schedule & --- & $\texttt{sqrt}$, $\lambda_{\mathrm{ref}}=0.25$, min mult.\ $0.40$ \\
LR base & --- & $10^{-3}$ \\
Loss reduction & --- & $\texttt{per\_dgp\_sum}$ \\
\midrule
Wall-clock & --- & 47h 23m 52.63s \\
\bottomrule
\end{tabular}
\end{table}

\subsection{Computational complexity}
\label{app:computational-complexity}

For the optimization-based frontier construction, a single optimization step at fixed $\lambda$ evaluates the spline flow, density ratio, divergence estimate, and outcome model on $Bmk$ latent samples, giving per-step cost $O(Bmk)$ (with $B$ the DGP batch size, $m$ queries per DGP, $k$ MC samples) plus one backward pass through the spline parameters. Total cost across the prior is
\[
    O\!\left(N_{\mathrm{DGP}} \cdot |\Lambda| \cdot T \cdot m \cdot k\right),
\]
where $T$ is the average number of optimizer steps per $\lambda$ (bounded by the schedule at $700$ for the KL run after the lambda-dependent step multiplier).
This cost is paid once during prior-data generation; at test time the trained PFN replaces the per-query optimization with a single forward pass.

Warm-starting reduces effective $T$ after the first $\lambda$: solved spline parameters initialize the next $\lambda$-solve, while Adam momentum is reset across $\lambda$ changes to avoid stale state from a different scalarized objective.

\subsection{Test-time inference}
\label{app:test-time-inference}

Test-time inference is a single transformer forward pass over a length-$(N+m)$ sequence, returning the parameters of both heads' Gaussian mixtures. Posterior predictive means are the GMM means; credible intervals are GMM quantiles, estimated by Monte Carlo sampling from the predicted mixtures.
APO, CATE, and ATE add only deterministic averaging and differencing on top of the CAPO PPDs (Section~\ref{app:aggregation}). 
\subsection{Reproducibility}
\label{app:reproducibility}

Run configurations appear in Tables~\ref{tab:pfn-config} and~\ref{tab:label-runs}, all randomness is seeded. Code release: \url{https://github.com/EmilJavurek/Amortizing-Causal-Sensitivity-Analysis-via-PFNs}.


\newpage
\section{Additional Experiments}
\label{app:additional-experiments}

\subsection{Warm-start ablation.}
\label{app:warm-start-ablation}

\begin{figure}[h]
    \centering
    \includegraphics[width=\linewidth]{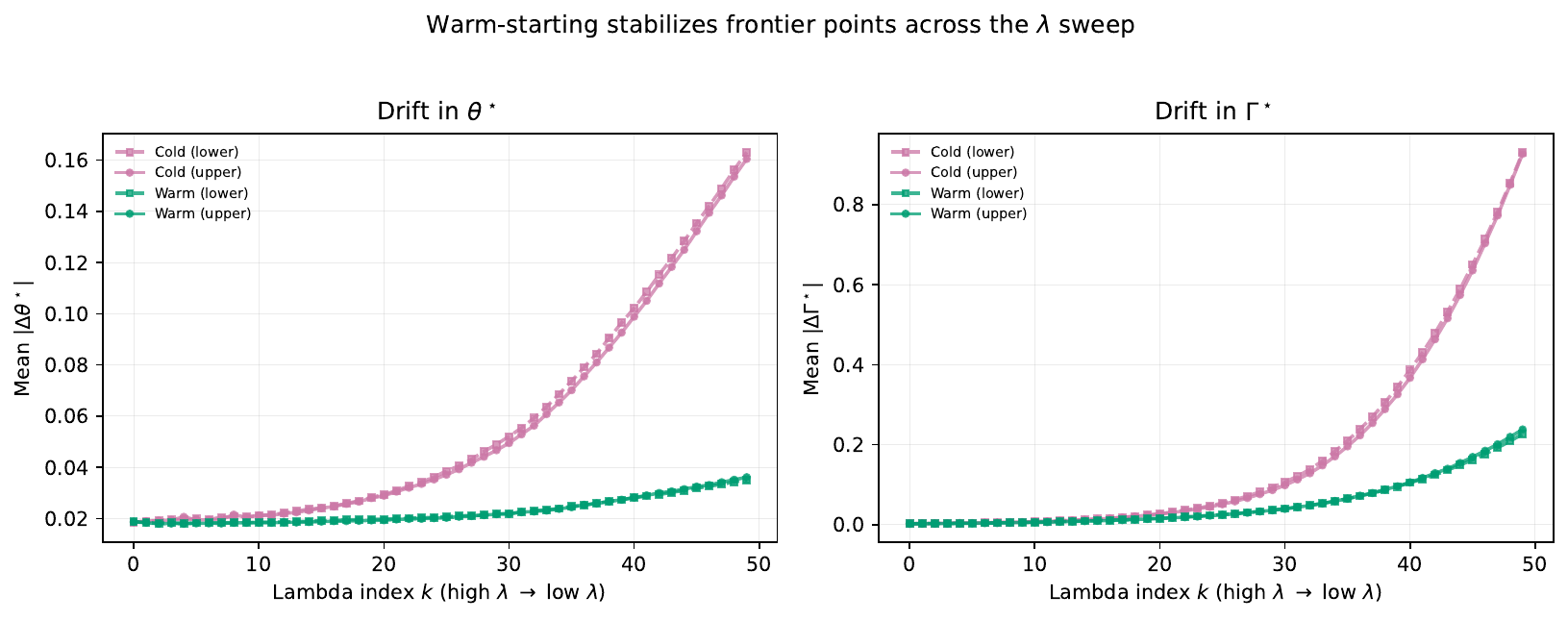}
    \caption{Warm start ablation. Drift in optimized causal query bound $\theta^star$ (left) and optimized sensitivity parameter $\Gamma^\star$ (right) along the $\lambda$-sweep ($k=0$ at $\lambda_{\mathrm{max}}=2.0$, $k=49$ at $\lambda_{\mathrm{min}}=0.08$.) measured against a high-budget reference (1000 steps). \textbf{Warm starting achieves lower regret solutions while $1.90\times$ faster.}}
    \label{fig:exp1-warmstart-solutiondrift}
\end{figure}

We use a one-dimensional rational-quadratic spline flow with 8 bins, tail bound 6.0, $k_{\mathrm{train}}=128$ Monte Carlo samples, and $k_{\mathrm{eval}}=1024$ final-evaluation samples.
Both early-stopped runs use the same optimizer settings: base maximum 350 steps, inverse-square-root $\lambda$-dependent step schedule with multiplier capped at 2.0, square-root $\lambda$-dependent learning-rate schedule with $\lambda_{\mathrm{ref}}=0.25$ and minimum multiplier 0.40, and early stopping after 100 minimum steps with checks every 25 steps, patience 3, absolute tolerance $2\times 10^{-4}$, and relative tolerance $5\times 10^{-4}$. Regret is computed relative to a warm-started fixed-budget reference with no early stopping and base maximum 1000 optimization steps. Experiment was run on an \texttt{Nvidia H200} GPU at $\sim 90\%$ utilization (no other processes running), taking in total approx. 1 hour.

\subsection{MSM foundation model}
\label{app:msm-foundation-model}

\subsubsection{Training diagnostics:}

\paragraph{Loss curves}

We monitor the total negative log-likelihood on the training and validation splits throughout optimization. The curves provide a basic diagnostic for convergence and potential overfitting. The selected checkpoint corresponds to the epoch with the lowest validation loss.

\begin{figure}[h]
    \centering
    \includegraphics[width=0.8\linewidth]{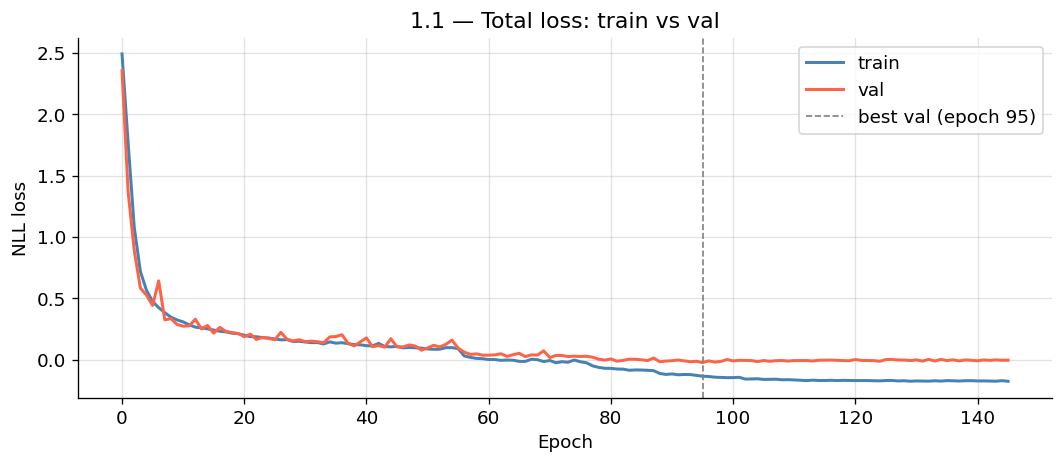}
    \caption{Training and validation negative log-likelihood decrease over epochs for the MSM foundation model.}
    \label{fig:appx-msm-total-loss}
\end{figure}

\paragraph{Performance}

We evaluate predictive performance using calibration and point-error diagnostics for the posterior predictive distributions of both bound heads. Coverage is reported for the central posterior predictive intervals and compared against the nominal levels. We also track the posterior predictive interval width to assess uncertainty contraction during training. Bias and RMSE summarize the accuracy of the predictive means.

\begin{figure}[h]
    \centering
    \includegraphics[width=0.8\linewidth]{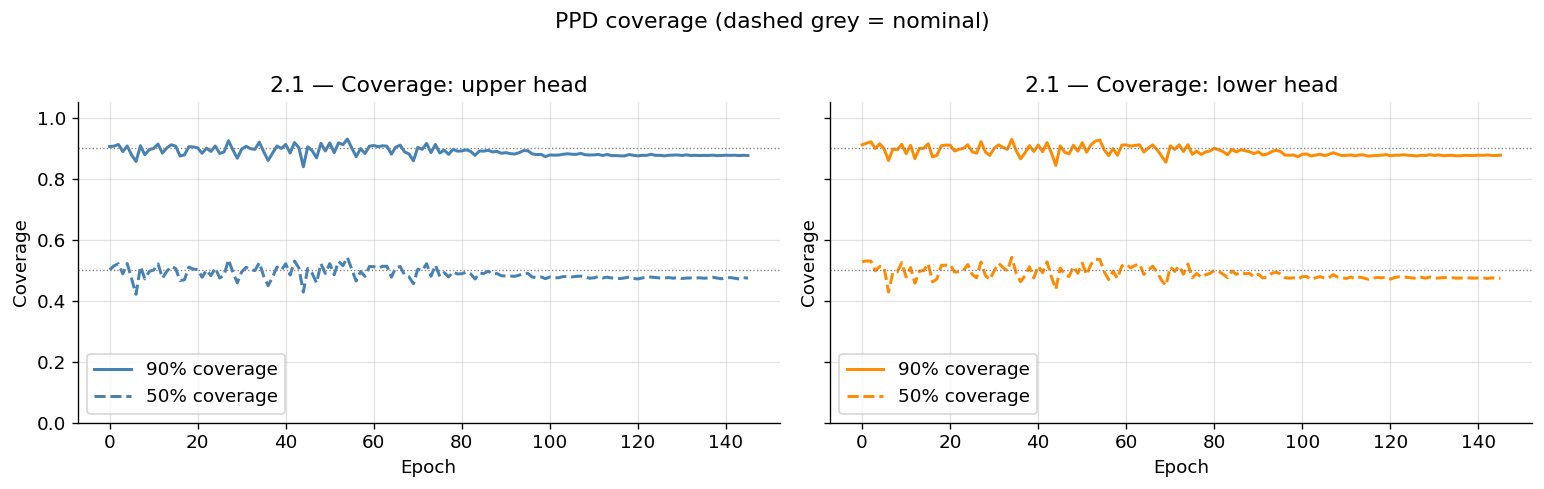}
    \includegraphics[width=0.8\linewidth]{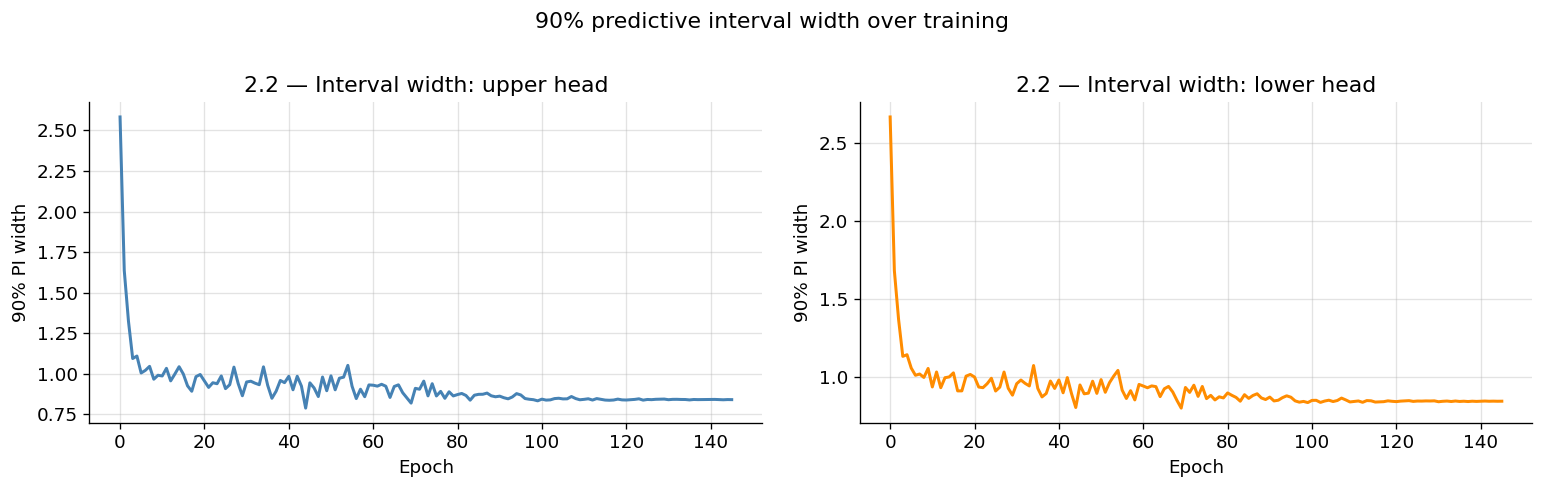}
    \includegraphics[width=0.8\linewidth]{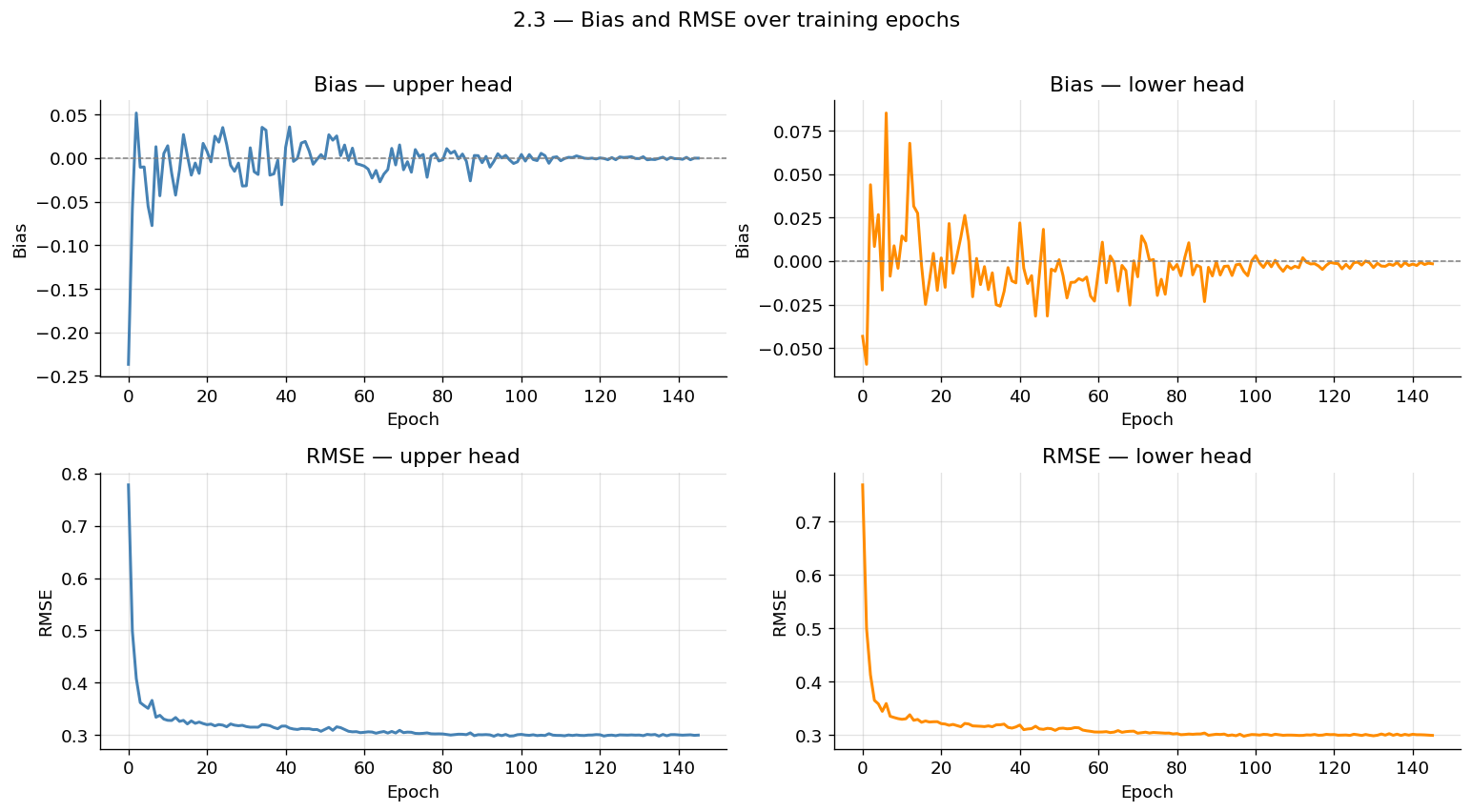}
    \caption{(Top) Posterior predictive coverage remains close to the nominal 90\% and 50\% levels for both bound heads. (Middle) The 90\% posterior predictive interval width contracts during training for both bound heads. (Bottom) Bias and RMSE decline during training for both bound heads.}
    \label{fig:appx-msm-coverage}
\end{figure}

\paragraph{Monotonicity regularization}

The MSM bounds should widen monotonically as the sensitivity level increases. We therefore track both the frequency and size of monotonicity violations during training. These diagnostics show whether the soft regularization term is active and whether violations persist after convergence. The penalty is compared against the total training loss to assess its relative contribution to optimization.

\begin{figure}[h]
    \centering
    \includegraphics[width=0.8\linewidth]{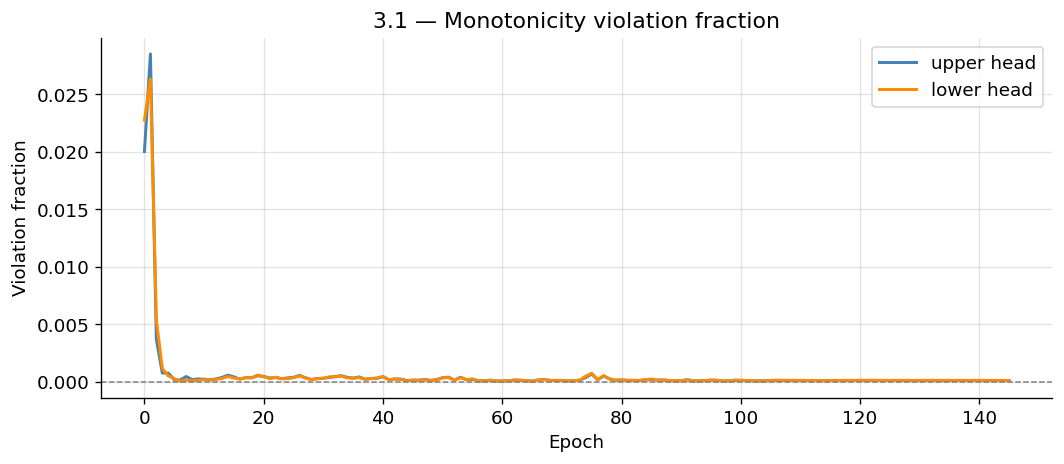}
    \includegraphics[width=0.8\linewidth]{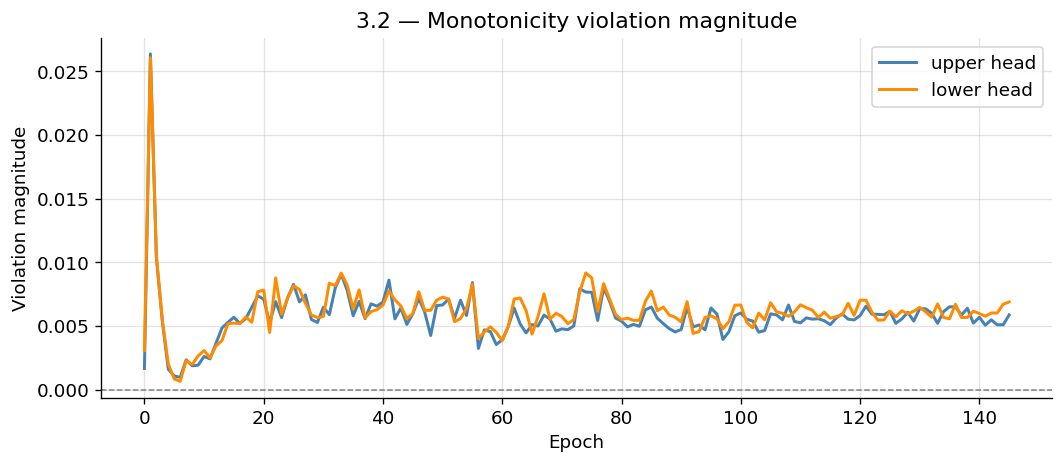}
    \includegraphics[width=0.8\linewidth]{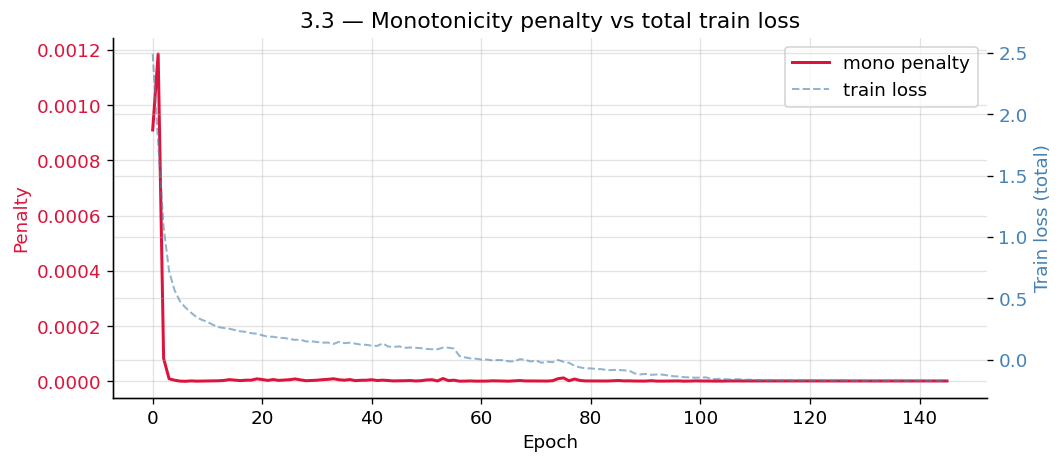}
    \caption{(Top) The fraction of monotonicity violations rapidly approaches zero for both bound heads. (Middle) The average monotonicity violation magnitude remains small after the initial training phase. (Bottom) The monotonicity penalty becomes negligible as the total training loss stabilizes.}
    \label{fig:appx-msm-violation-fraction}
\end{figure}

\section{Limitations and Broader Impact}
\label{app:limitations-impact}

\subsection{Limitations}
\label{app:limitations}

Our approach is limited by the scope of its prior, theory, and experiments. As with all prior-data fitted networks, performance depends on whether the synthetic SCM prior covers the data-generating mechanisms encountered at test time; substantial prior mismatch can lead to miscalibrated bounds. The theoretical construction targets generalized treatment sensitivity models and relies on convexity of the sensitivity divergence and linearity of the causal query, so it does not automatically extend to arbitrary sensitivity models or nonlinear estimands, even if those are uncommon. Empirically, we train and evaluate the foundation model primarily for the marginal sensitivity model, with an additional KL frontier-construction ablation, rather than exhaustively validating all GTSMs. For models without closed-form bounds, label generation remains an optimization-based approximation subject to finite compute, Monte Carlo noise, local optima, and normalizing-flow capacity. Finally, the current implementation focuses on binary treatments and scalar outcomes; broader treatment spaces and multivariate outcomes may require substantially more expensive label generation.

\subsection{Broader Impact}
\label{app:broader-impact}

This work aims to make causal sensitivity analysis cheaper and more systematic by amortizing bound computation across datasets, queries, treatment arms, and sensitivity levels. A positive impact is that researchers may be more likely to report sensitivity bounds rather than overconfident point estimates in observational studies, especially in domains such as medicine, public policy, and the social sciences. The main risks are misuse and overinterpretation: the bounds are only meaningful relative to the chosen sensitivity model, sensitivity level, and training prior, and should not be treated as a certificate of causal validity. Because amortization makes large-scale causal screening easier, users could also selectively report favorable sensitivity levels or subgroups. Responsible use therefore requires reporting full sensitivity curves, documenting assumptions, using domain expertise when choosing sensitivity ranges, and applying additional scrutiny in high-stakes settings.



\end{document}